\pdfoutput=1
\documentclass{article} 
\usepackage{microtype}
\usepackage{graphicx}
\usepackage{subcaption}
\usepackage{booktabs} 
\usepackage{algorithm}
\usepackage{algorithmic}
\renewcommand{\algorithmicrequire}{\textbf{Input:}}
\renewcommand{\algorithmicensure}{\textbf{Output:}}
  
\usepackage{graphicx}
\usepackage{tikz}
\usetikzlibrary{positioning, arrows.meta}

\usepackage{hyperref}
\usepackage[preprint]{icml2026}

\usepackage{amsmath}
\usepackage{amsfonts}
\usepackage{amssymb}
\usepackage{mathtools}
\usepackage{amsthm}
\usepackage{multirow}

\usepackage{tabularx}
\usepackage[table]{xcolor}
\definecolor{Gray}{gray}{0.9}

\usepackage[capitalize,noabbrev]{cleveref}

\usepackage{enumitem}

\theoremstyle{plain}
\newtheorem{theorem}{Theorem}[section]

\newtheorem{lemma}[theorem]{Lemma}

\theoremstyle{definition}
\newtheorem{definition}[theorem]{Definition}
\newtheorem{assumption}{Assumption}
\theoremstyle{remark}
\newtheorem{remark}[theorem]{Remark}

\renewcommand{\hat}{\widehat}

\usepackage[textsize=tiny]{todonotes}

\icmltitlerunning{HyPAC:\@ Cost-Efficient LLMs–Human Hybrid Annotation with PAC Error Guarantees}

\begin{document}

\twocolumn[
  \icmltitle{HyPAC:\@ Cost-Efficient LLMs–Human Hybrid Annotation \\ with PAC Error Guarantees}



  \icmlsetsymbol{equal}{*}

  \begin{icmlauthorlist}
    \icmlauthor{Hao Zeng}{equal,1}
    \icmlauthor{Huipeng Huang}{equal,1}
    \icmlauthor{Xinhao Qu}{2}
    \icmlauthor{Jianguo Huang}{3}
    \icmlauthor{Bingyi Jing}{4,1}
    \icmlauthor{Hongxin Wei}{1}
  \end{icmlauthorlist}

  \icmlaffiliation{1}{Department of Statistics and Data Science, Southern University of Science and Technology, China}
  \icmlaffiliation{2}{Department of Statistics, University of California at Riverside, USA}
  \icmlaffiliation{3}{College of Computing and Data Science, Nanyang Technological University, Singapore}
  \icmlaffiliation{4}{School of Artificial Intelligence, The Chinese University of Hong Kong, Shenzhen, China}

  \icmlcorrespondingauthor{Hongxin Wei}{weihx@sustech.edu.cn}

  \icmlkeywords{Data Annotation, Large Language Models, Hybrid Labeling, PAC Learning, Cost-Efficient Annotation, Uncertainty Quantification, Human-AI Collaboration}
  \vskip 0.3in
]



\printAffiliationsAndNotice{\icmlEqualContribution} 

\begin{abstract}
Data annotation often involves multiple sources with different cost-quality trade-offs, such as fast large language models (LLMs), slow reasoning models, and human experts.
In this work, we study the problem of routing inputs to the most cost-efficient annotation source while controlling the labeling error on test instances.
We propose \textbf{HyPAC}, a method that adaptively labels inputs to the most cost-efficient annotation source while providing distribution-free guarantees on annotation error.
HyPAC calibrates two decision thresholds using importance sampling and upper confidence bounds, partitioning inputs into three regions based on uncertainty and routing each to the appropriate annotation source.
We prove that HyPAC achieves the minimum expected cost with a probably approximately correct (PAC) guarantee on the annotation error, free of data distribution and pre-trained models.
Experiments on common benchmarks demonstrate the effectiveness of our method, reducing the annotation cost by 78.51\% while tightly controlling the annotation error.
\end{abstract}

\section{Introduction}
Data annotation is an important step in building machine learning systems~\citep{snow2008cheap,gligoric2024can}.
In practice, an annotation system generally involves multiple sources~\citep{qin2025crowdagent,wang2024humanllm} with different cost-quality trade-offs~\citep{yue2023large,dekoninck2025unified}, 
such as fast but less accurate large language models (LLMs), 
slow but more accurate reasoning models with extended inference (e.g., chain-of-thought prompting~\citep{chen2025reasoning,xu2025chain} 
or test-time compute~\citep{snell2024scaling,zhou2025theoretical}), 
and human annotators who provide the most accurate labels but at the highest cost~\citep{li2023coannotating}. 
Selecting the appropriate annotation source for each input is therefore crucial for achieving cost-efficient annotations while maintaining quality.

In the literature, there are many heuristic methods designed for \textit{hybrid annotation}
~\citep{yue2023large,dekoninck2025unified} with the core idea:
Route easy examples to fast, cheap models and send difficult examples to more accurate but expensive sources, such as stronger models or human annotators.
However, existing hybrid labeling approaches often rely on heuristics to select annotation sources, without formal guarantees on the error rate~\citep{li2023coannotating}.
This limitation raises a fundamental issue:
\vspace{-0.5em}
\begin{quote}
\textit{How to design hybrid annotation that balances cost and quality with provable error control?}
\end{quote}
\vspace{-0.5em}

In this paper, we introduce \textbf{HyPAC} (\textbf{Hy}brid annotation with \textbf{PAC} guarantees), a novel method for cost-efficient multi-source annotation with formal error guarantees.
We formulate hybrid annotation as a constrained cost optimization problem: minimize annotation cost while controlling the error rate below a user-specified tolerance.
HyPAC uses importance sampling and upper confidence bounds to set decision thresholds for routing examples to different labeling sources (weak models, strong models, or humans), ensuring error control with high probability while achieving the minimum expected annotation cost.
This provides practitioners with a principled way to optimize the cost-quality trade-off across the full spectrum of available annotators.

Theoretically, we provide formal guarantees for HyPAC in terms of both error control and cost efficiency.
First, we establish that HyPAC satisfies the ($\epsilon$, $\alpha$)-PAC annotation guarantee (Definition~\ref{def:pac-annotation}, Theorem~\ref{thm:main}): with probability at least $1 - \alpha$, the selected thresholds control the error rate below the user-specified tolerance $\epsilon$, and this guarantee is distribution-free.
Second, we show that HyPAC achieves the minimum cost among all threshold-based methods that satisfy the same PAC constraint (Theorem~\ref{thm:cost-optimal}).
Both results rely on the monotonicity of risk and cost functions in the threshold parameters (Lemmas~\ref{lem:monotonicity} and \ref{lem:cost-monotonicity}), which enables fixed-sequence testing without multiple testing correction.

We evaluate the effectiveness of HyPAC across diverse benchmark datasets including general tasks (MMLU-Redux~\citep{gema2025are}), math reasoning tasks (MATH-500~\citep{lightman2023lets}, MATH-L5~\citep{hendrycks2021measuring}, and Zebra-Logic~\citep{lin2025zebralogic}), and coding tasks ($\text{HumanEval}^+$~\citep{liu2023your}).
The results demonstrate that HyPAC significantly reduces the annotation cost compared to expert-only annotation while achieving the target error control.
For example, on MATH-500 with Qwen3-4B-Instruct-2507~\citep{yang2025qwen3} as the non-thinking model and Qwen3-4B-Thinking-2507 as the thinking model, HyPAC reduces the annotation cost by 78.51\%, while controlling the annotation error below 5\%.
Through ablation studies, we show that HyPAC outperforms baseline methods in cost and labeling error, the logits-based uncertainty score is more reliable than the verbalized uncertainty score, and our method is effective under various cost functions and UCBs. 
We provide the code for reproducing our main experiments in this anonymous repository: \href{https://anonymous.4open.science/r/HyPAC-B5AD}{https://anonymous.4open.science/r/HyPAC-B5AD}.

Our contributions are summarized:
\begin{enumerate}[itemsep=0pt]
    \item We formulate hybrid labeling as a \textbf{constrained cost optimization problem} with multiple labeling sources, and provide a provable framework for controlling the labeling error rate.
    \item We propose \textbf{HyPAC}, a novel practical method that calibrates decision thresholds using importance sampling and upper confidence bounds. 
    HyPAC is agnostic to the score function, supports \textbf{input-dependent costs}, and could \textbf{extend to multiple labeling sources}.
    \item We provide \textbf{distribution-free guarantees} on both error control and cost optimality: the error rate stays below a user-specified tolerance with high probability, and HyPAC achieves the minimum cost among all threshold-based methods with the same constraint.
\end{enumerate}

\textbf{Notation.}
Let $\mathcal{X}$ denote the input space and $\mathcal{Y}$ denote the label space.
We denote by $\mathcal{D}$ the distribution of \((x,y)\).
We use $\tilde{f}_1: \mathcal{X} \to \mathcal{Y}$ for the non-thinking model, $\tilde{f}_2: \mathcal{X} \to \mathcal{Y}$ for the thinking model, and $y$ for the human annotation (ground truth).
The cost functions of these three sources are denoted by $c_1(x)$, $c_2(x)$, and $c_h(x)$, respectively, where each cost may depend on the input $x$.
The score function $U: \mathcal{X} \to [0, 1]$ measures the uncertainty of the annotation system for input $x$.
The calibration set is denoted by $\mathcal{S}_{\text{cal}} = \{(x_i, y_i)\}_{i=1}^m$.
When $\mathcal{S}_{\text{cal}}$ appears under an expectation or probability operator, it refers to the randomness induced by the calibration dataset.

\section{Hybrid annotation}
\label{sec:preliminaries}

In this section, we present the problem of hybrid annotation for large language models and introduce the key concepts required for our method.
While our framework applies to general multi-source settings, we focus on a three-source setup in this paper for clarity: a non-thinking model, a thinking model, and human annotation. 
The framework can be easily extended to more sources (see Appendix~\ref{sec:extension}).

We consider a \textit{hybrid annotation} task where input samples $x \in \mathcal{X}$ need to be labeled with ground truth label $y \in \mathcal{Y}$, i.e., human annotation.
In this paper, we use non-thinking models as examples of fast models and thinking models as examples of slow models:
\begin{itemize}[itemsep=0pt]
    \item \textbf{Non-thinking model} $\tilde{f}_1: \mathcal{X} \to \mathcal{Y}$: a fast LLM using standard inference without extended reasoning, which is cheap but less accurate.
    For example, Qwen3-4B-Instruct generates responses quickly with lower computational cost.
    \item \textbf{Thinking model} $\tilde{f}_2: \mathcal{X} \to \mathcal{Y}$: a slow but more powerful LLM using extended reasoning (e.g., chain-of-thought), which is more accurate but more expensive.
    For example, Qwen3-4B-Thinking produces longer reasoning chains and requires more tokens per response.
    \item \textbf{Human annotation} $y$: labels from human experts. In most cases, we treat human labels as ground truth.
\end{itemize}

The goal of hybrid annotation is to design a decision rule \(\hat f\) that selects the annotation source for each input $x$, such that the overall annotation error rate is controlled while minimizing the expected annotation cost.
Formally, we define the annotation error as 
\begin{equation}
    \label{eq:anno_error}
R(\hat f) = \mathbb{E}_{(x,y) \sim \mathcal{D}}[\ell(y, \hat f(x))],
\end{equation}
where $\ell: \mathcal{Y} \times \mathcal{Y} \to \mathbb{R}_+$ is a loss function (e.g., 0-1 loss for classification, semantic loss for text generation).
Given a cost function \(C(\hat f, x)\), the goal is to solve:
\begin{equation}\label{eq:min_cost}
\min_{\hat f} \mathbb{E}_{x\sim \mathcal D_x}C(\hat f, x) \text{ s.t. }  R(\hat f)  \leq \epsilon, 
\end{equation}
where $\epsilon$ is the target error threshold.
We assume the quality of annotation sources follows: human annotation has the highest quality, followed by the thinking model, then the non-thinking model.

However, existing hybrid labeling methods often rely on heuristics for routing decisions~\citep{li2023coannotating,wang2024humanllm,li2024humanllm, yuan2025case}, without formal guarantees on the resulting error rate, which may lead to unpredictable performance in downstream tasks.
To address this challenge, we formalize the hybrid annotation problem with strict error rate control, providing formal guarantees on annotation quality while minimizing cost and enabling practitioners to meet pre-specified error thresholds reliably.

\section{HyPAC}
\label{sec:method}

In this section, we present \textbf{Hy}brid annotation with \textbf{PAC} guarantees (HyPAC), our method for cost-efficient hybrid annotation with rigorous labeling quality control guarantees.
We first define the PAC annotation problem and show how to express it as a constrained optimization over threshold parameters.
We then present our threshold calibration method, which uses importance sampling and upper confidence bounds to identify valid threshold pairs that satisfy the PAC constraint while minimizing annotation cost.

\subsection{PAC annotation}
First, we consider a score $U: \mathcal{X} \to [0, 1]$ that measures the uncertainty of the annotation system for input $x$.
Given two thresholds $(u_1, u_2)$ with $0 \le u_1 \le u_2 \le 1$, we define the decision rule \(\hat f = \) $T_{u_1,u_2}: \mathcal{X} \to \mathcal{Y}$ as:
\begin{equation*}
T_{u_1,u_2}(x) = \begin{cases}
\tilde{f}_1(x) & \text{if } U(x) \le u_1 \\
\tilde{f}_2(x) & \text{if } u_1 < U(x) \le u_2 \\
y & \text{if } U(x) > u_2
\end{cases}
\end{equation*}
Figure~\ref{fig:multi-level-example} illustrates this three-tier routing mechanism.

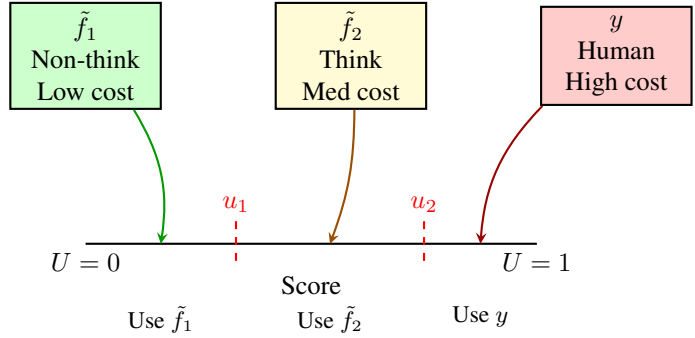
\begin{figure*}[t]
\centering
\begin{tikzpicture}[
  node distance=0.8cm and 1.5cm,
  modelbox/.style={rectangle, draw, thick, minimum width=2cm, minimum height=1.2cm, align=center},
  arrow/.style={->, thick, >=stealth},
  threshold/.style={dashed, red, thick}
]

\node[modelbox, fill=green!20] (f1) {$\tilde{f}_1$\\Non-think\\Low cost};
\node[modelbox, fill=yellow!20, right=of f1] (f2) {$\tilde{f}_2$\\Think\\Med cost};
\node[modelbox, fill=red!20, right=of f2] (f3) {$y$\\Human\\High cost};

\draw[thick] (0,-2.5) -- (6,-2.5);
\node[below] at (0,-2.5) {$U=0$};
\node[below] at (6,-2.5) {$U=1$};
\node[below] at (3,-2.8) {Score};

\draw[threshold] (2,-2.2) -- (2,-2.8);
\draw[threshold] (4.5,-2.2) -- (4.5,-2.8);
\node[above, red] at (2,-2.2) {$u_1$};
\node[above, red] at (4.5,-2.2) {$u_2$};

\node[below] at (1,-3.2) {\small Use $\tilde{f}_1$};
\node[below] at (3.25,-3.2) {\small Use $\tilde{f}_2$};
\node[below] at (5.25,-3.2) {\small Use $y$};

\draw[arrow, green!60!black] (f1) to[bend left=20] (1,-2.5);
\draw[arrow, orange!60!black] (f2) to[bend left=10] (3.25,-2.5);
\draw[arrow, red!60!black] (f3) to[bend right=20] (5.25,-2.5);

\end{tikzpicture}
\caption{\textbf{Hybrid annotation with three sources and two thresholds.}
The routing rule uses score $U(x)$ to select the appropriate annotation source: when $U(x) \le u_1$, use the non-thinking model $\tilde{f}_1$; when $u_1 < U(x) \le u_2$, use the thinking model $\tilde{f}_2$; when $U(x) > u_2$, use human annotation.}
\label{fig:multi-level-example}
\end{figure*}

Then we rewrite its annotation error as 
\[ 
R(\hat f) = R(u_1, u_2) = \mathbb{E}_{(x,y) \sim \mathcal{D}}[\ell(y, T_{u_1,u_2}(x))]
\]
For cost, let $c_1(x)$, $c_2(x)$, and $c_h(x)$ be the cost functions for using the non-thinking model, thinking model, and human annotation on input $x$, respectively.
In practice, these costs often depend on the input, for example, the number of tokens generated by the model. 
Under costs of sources \(c_1(x), c_2(x),\) and \(c_h(x)\), the cost function is:
\begin{equation*}
\begin{aligned}
C(u_1, u_2, x) = \; & c_1(x) \cdot \mathbf{1}\{U(x) \le u_1\} \\ 
& + c_2(x) \cdot \mathbf{1}\{u_1 < U(x) \le u_2\} \\
& + c_h(x) \cdot \mathbf{1}\{U(x) > u_2\}
\end{aligned}
\end{equation*}
We assume that the costs satisfy $c_1(x) \leq c_2(x) \leq c_h(x)$ for almost all $x$, reflecting the cost hierarchy of the three sources. 
Then the goal~Eq.~\eqref{eq:min_cost} becomes, 
\begin{equation}
\label{eq:optimization-ideal}
\begin{aligned}
\min_{u_1, u_2} \quad & \mathbb{E}_{x \sim \mathcal{D}_x}[C(u_1, u_2, x)] \\
\text{subject to} \quad & R(u_1, u_2) \le \epsilon, \quad 0 \le u_1 \le u_2 \le 1
\end{aligned}
\end{equation}
Directly solving this optimization~Eq.~\eqref{eq:optimization-ideal} is impossible because we do not know the true risk $R(u_1, u_2)$ without access to the ground truth labels $y$ for all inputs.
Thus, we consider a relaxed form where the constraint holds in a PAC style.
Given a calibration set $\mathcal{S}_{\text{cal}}$, we seek thresholds $(\hat{u}_1, \hat{u}_2)$ that minimize the expected cost $\mathbb{E}_{x \sim \mathcal{D}_x}[C(\hat{u}_1, \hat{u}_2, x)]$ while satisfying the ($\epsilon$, $\alpha$)-PAC annotation as follows.

\begin{definition}[($\epsilon$, $\alpha$)-PAC annotation]
\label{def:pac-annotation}
An annotation procedure using thresholds $({u}_1, {u}_2)$ is ($\epsilon$, $\alpha$)-PAC if
\begin{equation}
\label{eq:pac-annotation}
\mathbb{P}_{\mathcal{S}_{\text{cal}}} \left( R({u}_1, {u}_2) \le \epsilon \right) \ge 1 - \alpha,
\end{equation}
where the risk $R({u}_1, {u}_2)$ is evaluated on the test distribution $\mathcal{D}$, and the probability is over the calibration set $\mathcal{S}_{\text{cal}}$.\textsl{}
\end{definition}
Since multiple threshold pairs may satisfy this condition, we call the set of all such pairs the \textit{($\epsilon$, $\alpha$)-PAC feasible region}.
It is inspired by the PAC efficient~\citep{zeng2025pac}, which considers a two-level reasoning system with fast and slow models. 
We extend it to multiple annotators with a human annotator.
We aim to find thresholds that satisfy the PAC guarantee while minimizing the annotation cost:
\begin{equation}
\label{eq:optimization-pac}
\begin{aligned}
\min_{u_1, u_2} \quad & \mathbb{E}_{x \sim \mathcal{D}_x}[C(u_1, u_2, x)] \\
\text{s.t. } \quad & \mathbb{P}_{\mathcal{S}_{\text{cal}}} \left( R({u}_1, {u}_2) \le \epsilon \right) \ge 1 - \alpha, \\ & 0 \le u_1 \le u_2 \le 1.
\end{aligned}
\end{equation}

This framework naturally extends to more than three annotation sources by adding more thresholds to partition the confidence space into more regions, where each region is routed to a different annotation source. We discuss this extension in detail in Appendix~\ref{sec:extension}.

\subsection{Threshold calibration}
\label{sec:calibration}

The main challenge in solving this problem~\eqref{eq:optimization-pac} is that we cannot directly check whether a threshold pair $(u_1, u_2)$ satisfies the PAC condition, because the true risk $R(u_1, u_2)$ is unknown without access to ground truth labels for all inputs.
To address this, we use a valid upper confidence bound (UCB) on the risk: a bound that holds with high probability can serve as a substitute for the true risk to identify valid thresholds.
In light of this, our method follows a two-step approach: first, we identify all threshold pairs that lie in the ($\epsilon$, $\alpha$)-PAC feasible region using the UCB; then, among valid thresholds, we select the one with the lowest cost.

\paragraph{Importance sampling for unbiased risk estimation.}
We use importance sampling~\citep{owen2000safe,tokdar2010importance} to obtain unbiased risk estimates.
For each example $x_i$ in the calibration set, we query the ground truth label $y_i$ with probability $p_i$.
The importance-weighted risk estimator is:
\begin{equation} 
\label{eq:is-risk}
\hat{R}_{\text{IS}}(u_1, u_2) = \frac{1}{m} \sum_{i=1}^m \frac{Z_i}{p_i} \cdot \ell(y_i, T_{u_1,u_2}(x_i))
\end{equation}
where $Z_i \sim \text{Bernoulli}(p_i)$ indicates whether we queried the ground truth for example $i$.
This estimator is unbiased: $\mathbb{E}[\hat{R}_{\text{IS}}(u_1, u_2)] = R(u_1, u_2)$.
In practice, we often use uniform sampling with $p_i = p$ for all $i$.
The choice of $p$ involves a trade-off: a larger $p$ gives more accurate estimates but requires more ground truth queries, while a smaller $p$ reduces calibration cost but increases variance.

\paragraph{UCB construction.}
Given the importance-weighted risk estimates, we construct an upper confidence bound (UCB) for each threshold pair $(u_1, u_2)$.
Define the importance-weighted loss for each sample as:
\begin{equation*}
W_i(u_1, u_2) = \frac{Z_i}{p_i} \cdot \ell(y_i, T_{u_1,u_2}(x_i))
\end{equation*}
so that $\hat{R}_{\text{IS}}(u_1, u_2) = \frac{1}{m} \sum_{i=1}^m W_i(u_1, u_2)$.
Since $\{W_i\}_{i=1}^m$ are i.i.d.\ random variables with $\mathbb{E}[W_i] = R(u_1, u_2)$, by the Central Limit Theorem (CLT), the sample mean is asymptotically normal:
\begin{equation*}
\frac{\hat{R}_{\text{IS}}(u_1, u_2) - R(u_1, u_2)}{\hat{\sigma}_W(u_1, u_2) / \sqrt{m}} \xrightarrow{d} \mathcal{N}(0, 1)
\end{equation*}
where $\hat{\sigma}_W(u_1, u_2)$ is the sample standard deviation of $\{W_i\}$.
This yields the CLT-based UCB:
\begin{equation}
\label{eq:ucb}
\hat{L}_{u_1,u_2}(\alpha) = \hat{R}_{\text{IS}}(u_1, u_2) + z_{1-\alpha} \frac{\hat{\sigma}_W(u_1, u_2)}{\sqrt{m}}
\end{equation}
where $z_{1-\alpha}$ is the $(1-\alpha)$-quantile of the standard normal distribution.
Other valid choices (e.g., Hoeffding-type or empirical Bernstein bounds) can be used to achieve distribution-free finite-sample guarantees.

\paragraph{Threshold selection.}
Given the UCBs for all threshold pairs, we select the pair that minimizes cost while ensuring the error rate is controlled.
We estimate the expected cost using the empirical cost on the calibration set:
\begin{equation*}
\hat{C}(u_1, u_2) = \frac{1}{m} \sum_{i=1}^m C(u_1, u_2, x_i)
\end{equation*}
The optimization problem then becomes:
\begin{equation}
\label{eq:threshold-selection}
\begin{aligned}
(\hat{u}_1, \hat{u}_2) = &  \arg\min_{u_1, u_2} \quad  \hat{C}(u_1, u_2) \\
\text{s.t. } \quad & \hat{L}_{u_1,u_2}(\alpha) \le \epsilon \\
& 0 \le u_1 \le u_2 \le 1
\end{aligned}
\end{equation}

We then solve the optimization problem in Eq.~\eqref{eq:threshold-selection} by searching over a given threshold grid.
This calibration procedure ensures that the selected thresholds control the error rate on future test data with high probability, while minimizing the annotation cost.
We illustrate it in Figure~\ref{fig:ucb-illustration}.

\begin{figure}[ht]
\centering
\begin{tikzpicture}[
  arrow/.style={->, thick, >=stealth},
  feasible/.style={fill=green!20, draw=green!60!black, thick},
  ucbboundary/.style={blue!70!black, thick},
  trueboundary/.style={blue!40, thick, dashed}
]
\draw[arrow] (0,0) -- (5.5,0) node[right] {$u_1$};
\draw[arrow] (0,0) -- (0,5.5) node[above] {$u_2$};

\fill[gray!10] (0,0) -- (5,5) -- (0,5) -- cycle;

\fill[feasible] (0,0) -- (0,3.5) -- (1.5,3.5) -- (2.5,2.5) -- cycle;

\draw[ucbboundary] (0,3.5) -- (1.5,3.5) -- (2.5,2.5);

\draw[trueboundary] (0,4.2) -- (2.2,4.2) -- (3.2,3.2);

\draw[thick, gray] (0,0) -- (5,5);
\node[gray, rotate=45] at (4.2,4.5) {\small $u_1 = u_2$};

\fill[red] (1.5,3.5) circle (3pt);
\node[below right, red] at (1.5,3.8) {$(\hat{u}_1, \hat{u}_2)$};


\node at (0.85,2) {\small \textbf{feasible area}};

\node[below] at (5,0) {\small $1$};
\node[left] at (0,5) {\small $1$};

\draw[blue!70!black, thick] (2.8,1.6) -- (3.4,1.6);
\node[right] at (3.5,1.6) {\scriptsize UCB: $\hat{L}_{u_1,u_2} = \epsilon$};

\draw[blue!40, thick, dashed] (2.8,1.2) -- (3.4,1.2);
\node[right] at (3.5,1.2) {\scriptsize Risk boundary: $R(u_1,u_2) = \epsilon$};

\fill[red] (3.1,0.8) circle (2pt);
\node[right] at (3.5,0.8) {\scriptsize UCB-based solution};

\end{tikzpicture}
\caption{
\textbf{Illustration of UCB-based threshold selection in the $(u_1, u_2)$ space.}
The green region shows the $(\epsilon, \alpha)$-PAC feasible region where $\hat{L}_{u_1,u_2}(\alpha) \le \epsilon$ and $u_1 \le u_2$.
The solid blue line is the UCB boundary, and the dashed blue line is the true risk boundary.
Since the UCB upper bounds the true risk, the UCB boundary lies inside the true risk boundary.
The red point marks the optimal solution selected by HyPAC, which maximizes thresholds (i.e., minimizing cost) within the feasible region.}
\label{fig:ucb-illustration}
\end{figure}
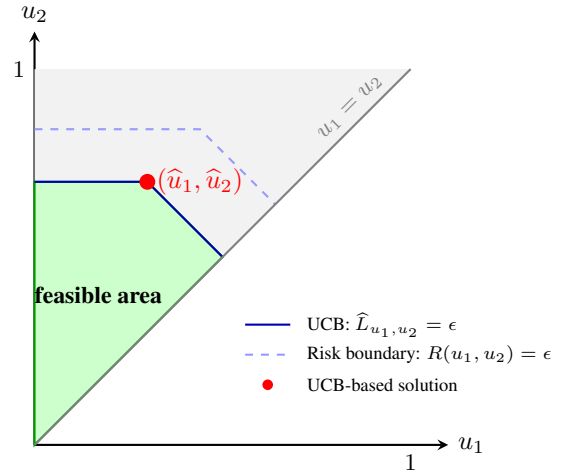

\paragraph{Deployment.}
Once we have calibrated the thresholds $(\hat{u}_1, \hat{u}_2)$, the deployment phase is straightforward.
For each test example $x$, we compute the score $U(x)$ and route the example to the appropriate annotation source based on the calibrated thresholds.

The complete procedure, including both calibration and deployment phases, is summarized in Algorithm~\ref{alg:hypac}.
Next, we provide a theoretical analysis showing that this method achieves PAC annotation guarantees and cost optimality.

\begin{algorithm}[t!]
\caption{HyPAC}
\label{alg:hypac}
\begin{algorithmic}[1]
\REQUIRE Calibration set $\mathcal{S}_{\text{cal}} = \{x_i\}_{i=1}^m$, test example $x$, models $\tilde{f}_1$, $\tilde{f}_2$, score function $U$, error tolerance $\epsilon$, significance level $\alpha$, sampling probability $p$, threshold grid $\mathcal{G}$
\ENSURE Label $\hat{y}$ for test example $x$
\STATE \textit{// Calibration phase}
\FOR{$i = 1$ to $m$}
\STATE Compute score $U(x_i)$ and predictions $\tilde{f}_1(x_i)$, $\tilde{f}_2(x_i)$
\STATE Sample $Z_i \sim \text{Bernoulli}(p)$; if $Z_i = 1$, query ground truth $y_i$
\ENDFOR
\FOR{each $(u_1, u_2) \in \mathcal{G}$ with $u_1 \le u_2$}
    \STATE Compute importance-weighted risk estimate $\hat{R}_{\text{IS}}(u_1, u_2)$ via~Eq.~\eqref{eq:is-risk}
    \STATE Compute UCB $\hat{L}_{u_1,u_2}(\alpha)$ via~Eq.~\eqref{eq:ucb}
\ENDFOR
\STATE Compute $(\hat{u}_1, \hat{u}_2)$ via Eq.~\eqref{eq:threshold-selection}
\STATE \textit{// Deployment phase}
\STATE Compute score $U(x)$
\STATE {\bfseries Return} $T_{\hat{u}_1,\hat{u}_2}(x)$ 
\end{algorithmic}
\end{algorithm}

\section{Theoretical analysis}
\label{sec:theory}

In this section, we provide a theoretical analysis of the HyPAC method.
We first show that both the risk and cost functions are monotone in the threshold parameters (Section~\ref{sec:monotonicity}), which is key to establishing the PAC guarantee and cost optimality.
We then present our main results (Section~\ref{sec:pac-guarantee}): a PAC guarantee showing that the selected thresholds control the error rate (Theorem~\ref{thm:main}), i.e., PAC annotation, and a cost optimality result showing that HyPAC achieves the minimum cost among all PAC-valid solutions (Theorem~\ref{thm:cost-optimal}).

\subsection{Monotonicity}
\label{sec:monotonicity}

We now prove that the risk function $R(u_1, u_2)$ is monotone in both threshold parameters.
This result is crucial: it implies that to minimize cost while controlling risk, we should select the largest thresholds within the PAC feasible region.

\begin{lemma}[Risk monotonicity]
\label{lem:monotonicity}
The risk function $$R(u_1, u_2) = \mathbb{E}_{(x,y) \sim \mathcal{D}}[\ell(y, T_{u_1,u_2}(x))]$$ is monotone non-decreasing in both $u_1$ and $u_2$.
Specifically, for any fixed $u_2$ and $u_1 < u_1'$ with $u_1' \le u_2$, we have $R(u_1, u_2) \le R(u_1', u_2)$.
Similarly, for any fixed $u_1$ and $u_2 < u_2'$, we have $R(u_1, u_2) \le R(u_1, u_2')$.
\end{lemma}

The complete proof is given in Appendix~\ref{app:monotonicity}.
The monotonicity property ensures that the feasible region $\{(u_1, u_2) : R(u_1, u_2) \le \epsilon, 0 \le u_1 \le u_2 \le 1\}$ has a nested structure:
If $(u_1, u_2)$ is feasible, then any $(u_1', u_2')$ with $u_1' \le u_1$ and $u_2' \le u_2$ is also feasible.
This nested structure enables fixed-sequence testing without multiple testing correction, as we show in the proof of Theorem~\ref{thm:main}.

\begin{lemma}[Cost monotonicity]
\label{lem:cost-monotonicity}
Assuming the costs satisfy $c_1(x) \le c_2(x) \le c_h(x)$ (cost ordering), we have:
Both the expected cost $\mathbb{E}_{x \sim \mathcal{D}_x}[C(u_1, u_2, x)]$ and the empirical cost $\hat{C}(u_1, u_2) = \frac{1}{m}\sum_{i=1}^m C(u_1, u_2, x_i)$ (for any finite sample $\{x_i\}_{i=1}^m$) are monotone non-increasing in both $u_1$ and $u_2$.
Furthermore, if $c_1(x) < c_2(x) < c_h(x)$ strictly for almost all $x$, both the expected cost and empirical cost are strictly decreasing in both $u_1$ and $u_2$.
\end{lemma}

The proof is given in Appendix~\ref{app:cost-monotonicity}.
These two monotonicity results have important practical implications.
Risk monotonicity (Lemma~\ref{lem:monotonicity}) tells us that larger thresholds lead to higher risk (more errors).
Cost monotonicity (Lemma~\ref{lem:cost-monotonicity}) tells us that larger thresholds lead to lower cost (more examples routed to cheaper sources).


\begin{table*}[t]
\centering
\caption{\textbf{HyPAC achieves higher cost savings than both PAC labeling and CSE baselines while maintaining controlled annotation error.}
Results for annotation error and cost savings are reported at $\epsilon=0.05$ and $\alpha=0.05$ using the token-based cost function.
The non-thinking model $\tilde{f}_1$ is Qwen3-4B-Instruct-2507, and the thinking model $\tilde{f}_2$ is Qwen3-4B-Thinking-2507.
HyPAC is compared against three baselines: (1) PAC labeling with non-thinking models and experts ($\tilde{f}_1$ + Human), (2) PAC labeling with thinking models and experts ($\tilde{f}_2$ + Human), and (3) CSE.
\textbf{Bold} numbers indicate superior results.
}
\label{tab:method_comparison}
\renewcommand\arraystretch{1.1}
\resizebox{1.00\textwidth}{!}{
\setlength{\tabcolsep}{2.3mm}{
\begin{tabular}{ll|cccc|cccc}
\toprule
\multirow{2}{*}{Datasets}
& \multirow{2}{*}{Metric}
& \multicolumn{4}{c}{Logits-based Score}
& \multicolumn{4}{c}{Verbalized Score} \\
\cmidrule(lr){3-6} \cmidrule(lr){7-10}
&
& $\tilde{f}_1$ + $\tilde{f}_2$ + Human
& $\tilde{f}_1$ + Human
& $\tilde{f}_2$ + Human
& CSE
& $\tilde{f}_1$ + $\tilde{f}_2$ + Human
& $\tilde{f}_1$ + Human
& $\tilde{f}_2$ + Human
& CSE \\
\midrule

\multirow{2}{*}{MMLU-Redux}
& Error (\%)
& 4.09 $\pm$ 0.94 & 4.00 $\pm$ 1.20 & 3.97 $\pm$ 1.34 & 3.96 $\pm$ 4.70
& 2.97 $\pm$ 0.47 & 2.02 $\pm$ 0.14 & 0.60 $\pm$ 0.08 & 3.37 $\pm$ 2.58 \\
& Save (\%)
& \textbf{49.40} $\pm$ 5.11 & 48.97 $\pm$ 7.36 & 40.70 $\pm$ 7.72 & 28.01 $\pm$ 32.84
& \textbf{39.27} $\pm$ 0.38 & 37.64 $\pm$ 0.71 & -6.03 $\pm$ 0.22 & 36.06 $\pm$ 14.34 \\
\midrule

\multirow{2}{*}{MATH-500}
& Error (\%)
& 4.00 $\pm$ 1.71 & 3.39 $\pm$ 1.77 & 3.09 $\pm$ 1.54 & 4.14 $\pm$ 4.45
& 4.22 $\pm$ 2.99 & 3.41 $\pm$ 1.70 & 1.75 $\pm$ 2.20 & 4.40 $\pm$ 1.62 \\
& Save (\%)
& \textbf{78.51} $\pm$ 4.69 & 70.85 $\pm$ 12.04 & 67.20 $\pm$ 16.70 & 43.38 $\pm$ 51.14
& \textbf{76.94} $\pm$ 9.34 & 66.22 $\pm$ 21.17 & 16.62 $\pm$ 44.97 & 74.51 $\pm$ 1.98 \\
\midrule

\multirow{2}{*}{MATH-L5}
& Error (\%)
& 4.17 $\pm$ 1.34 & 3.46 $\pm$ 1.74 & 3.39 $\pm$ 1.56 & 4.20 $\pm$ 3.96
& 4.35 $\pm$ 8.30 & 2.88 $\pm$ 1.41 & 1.78 $\pm$ 1.66 & 4.32 $\pm$ 4.89 \\
& Save (\%)
& \textbf{72.55} $\pm$ 4.35 & 62.98 $\pm$ 15.22 & 43.49 $\pm$ 25.19 & 37.50 $\pm$ 46.17
& \textbf{70.94} $\pm$ 15.63 & 67.64 $\pm$ 22.32 & 21.83 $\pm$ 24.00 & 70.66 $\pm$ 32.48 \\
\midrule

\multirow{2}{*}{Zebra-Logic}
& Error (\%)
& 4.48 $\pm$ 1.30 & 3.86 $\pm$ 1.65 & 3.63 $\pm$ 1.34 & 4.58 $\pm$ 2.62
& 4.47 $\pm$ 6.27 & 1.34 $\pm$ 2.22 & 0.14 $\pm$ 0.15 & 4.34 $\pm$ 1.69 \\
& Save (\%)
& \textbf{46.10} $\pm$ 3.75 & 43.76 $\pm$ 6.34 & 45.39 $\pm$ 8.15 & 36.83 $\pm$ 22.26
& \textbf{36.03} $\pm$ 8.34 & 1.11 $\pm$ 19.50 & -22.17 $\pm$ 0.44 & 5.56 $\pm$ 15.42 \\
\midrule

\multirow{2}{*}{$\text{HumanEval}^{+}$}
& Error (\%)
& 3.25 $\pm$ 1.65 & 3.39 $\pm$ 3.09 & 2.42 $\pm$ 1.69 & 3.44 $\pm$ 3.34
& 3.36 $\pm$ 1.74 & 2.28 $\pm$ 3.24 & 1.69 $\pm$ 1.97 & 3.28 $\pm$ 1.77 \\
& Save (\%)
& \textbf{86.48} $\pm$ 5.35 & 62.13 $\pm$ 23.17 & 75.26 $\pm$ 24.46 & 67.51 $\pm$ 21.39
& 70.55 $\pm$ 34.50 & 30.49 $\pm$ 42.60 & 49.97 $\pm$ 43.01 & \textbf{87.46} $\pm$ 1.63 \\
\bottomrule

\end{tabular}
}
}
\vspace{-5pt}
\end{table*}

\subsection{PAC guarantee and cost optimal}
\label{sec:pac-guarantee}

We now state our main results.
We first introduce the key assumption on the UCB construction.

\begin{assumption}[UCB validity]
\label{assump:ucb}
For each threshold pair $(u_1, u_2)$ with $0 \le u_1 \le u_2 \le 1$ and any significance level $\alpha \in (0,1)$, the upper confidence bound $\hat{L}_{u_1,u_2}(\alpha)$ satisfies:
\begin{equation*}
\mathbb{P}_{\mathcal{S}_{\text{cal}}} \left( R(u_1, u_2) \le \hat{L}_{u_1,u_2}(\alpha) \right) \ge 1 - \alpha
\end{equation*}
where the probability is taken over the random draw of the calibration set $\mathcal{S}_{\text{cal}}$.
\end{assumption}

This assumption requires  UCB to provide valid coverage for the true risk.
In Section~\ref{sec:calibration}, we construct such a UCB.
The validity of this UCB can be established using various methods, such as the central limit theorem (CLT), Hoeffding's inequality, or Bernstein bounds; see Appendix~\ref{app:ucb-validity}.
Under the assumption, we could have the validity of  HyPAC:
\begin{theorem}[PAC guarantee]
\label{thm:main}
Let $(\hat{u}_1, \hat{u}_2)$ be the thresholds selected by Algorithm~\ref{alg:hypac} with error tolerance $\epsilon$ and significance level $\alpha$.
Under Assumption~\ref{assump:ucb}, and assuming independent calibration/test sets and fixed pre-trained models~\(\tilde{f}_1, \tilde{f}_2\), we have:
\begin{equation*}
\mathbb{P}_{\mathcal{S}_{\text{cal}}} \left( R(\hat{u}_1, \hat{u}_2) \le \epsilon \right) \ge 1 - \alpha
\end{equation*}
where the probability is taken over the random draw of the calibration set $\mathcal{S}_{\text{cal}}$.
\end{theorem}

The complete proof is given in Appendix~\ref{app:main-theorem}.
This theorem provides a rigorous guarantee that HyPAC controls the error rate with high probability.
The guarantee is \textbf{distribution-free}: it holds for any data distribution $\mathcal{D}$ and any pre-trained models $\tilde{f}_1$, $\tilde{f}_2$, without assumptions on the data structure or model quality.
We also provide a finite-sample guarantee for the empirical risk; see Appendix~\ref{app:empirical-risk}.

Beyond error control, HyPAC also achieves cost optimality among all methods that satisfy the PAC constraint.

\begin{theorem}[Cost optimality]
\label{thm:cost-optimal}
Let $\mathcal{S}_{\text{cal}}$ be a calibration set and $\mathcal{F}(\mathcal{S}_{\text{cal}}) = \{(u_1, u_2) : \hat{L}_{u_1,u_2}(\alpha) \le \epsilon, \; 0 \le u_1 \le u_2 \le 1\}$ be the PAC feasible region based on $\mathcal{S}_{\text{cal}}$.
Assuming the cost ordering $c_1(x) \le c_2(x) \le c_h(x)$, the threshold pair $(\hat{u}_1, \hat{u}_2)$ selected by Algorithm~\ref{alg:hypac} satisfies:
\begin{equation*}
\mathbb{E}_{x \sim \mathcal{D}_x}[C(\hat{u}_1, \hat{u}_2, x)] = \min_{(u_1, u_2) \in \mathcal{F}(\mathcal{S}_{\text{cal}})} \mathbb{E}_{x \sim \mathcal{D}_x}[C(u_1, u_2, x)],
\end{equation*}
\end{theorem}

The proof is given in Appendix~\ref{app:cost-optimal}.
This result shows that HyPAC is not only a PAC annotation but also cost-optimal given the calibration data.

\section{Experiments}
\label{sec:experiments}

In this section, we evaluate HyPAC on multiple benchmark datasets. 
Our experiments aim to answer two questions: 
(1) whether HyPAC achieves higher cost savings compared to existing provable methods with the same error control budgets; 
(2) whether HyPAC provides reliable error control compared to heuristic baselines. 

\subsection{Experimental setup}
\label{sec:exp-setup}

\paragraph{Models.}
We evaluate HyPAC on diverse open-source LLMs across different parameter scales.
For the main experiments, we employ Qwen3-4B-Instruct-2507~\citep{yang2025qwen3} as the non-thinking model and Qwen3-4B-Thinking-2507 as the thinking model.
In addition, we evaluate HyPAC using two alternative model pairs. 
Specifically, we use Llama-3.1-8B-Instruct~\citep{grattafiori2024llama} as the non-thinking model with DeepSeek-R1-Distill-Llama-8B~\citep{deepseek-ai2025deepseekr1} as the thinking model, and Qwen2.5-32B-Instruct~\citep{qwen2025qwen25} as the non-thinking model with DeepSeek-R1-Distill-Qwen-32B as the thinking model.
The sampling temperature and other hyperparameters for LLMs are detailed in Appendix~\ref{para:hyperparameter}.

\paragraph{Uncertainty scores.}
We adopt two complementary score functions:
a logits-based score for the white-box case and a verbalized score for the black-box case.
In particular, let $y = (y_1, y_2, \ldots, y_l)$ be a generated sequence of length $l$.
We define the logits-based score as:
\begin{equation*}
U_{\text{logits}}(x) = 1 - \frac{1}{l} \sum_{j=1}^l \mathbb{P}(y_j \mid y_1, \ldots, y_{j-1}, x)
\end{equation*}
where $\mathbb{P}(y_j \mid y_1, \ldots, y_{j-1}, x)$ is the conditional probability of token $y_j$.
This score measures the average confidence of the model across all generated tokens.
For the verbalized score, we let the model explicitly state its self-reported confidence.
In this study, we report the average confidence over 10 trials, and the prompts are listed in Table~\ref{tab:verbalized_score_prompt}.

\begin{figure*}
    \centering
    \begin{subfigure}[b]{0.49\textwidth}
        \centering
        \includegraphics[height=5.0cm,width=8cm]{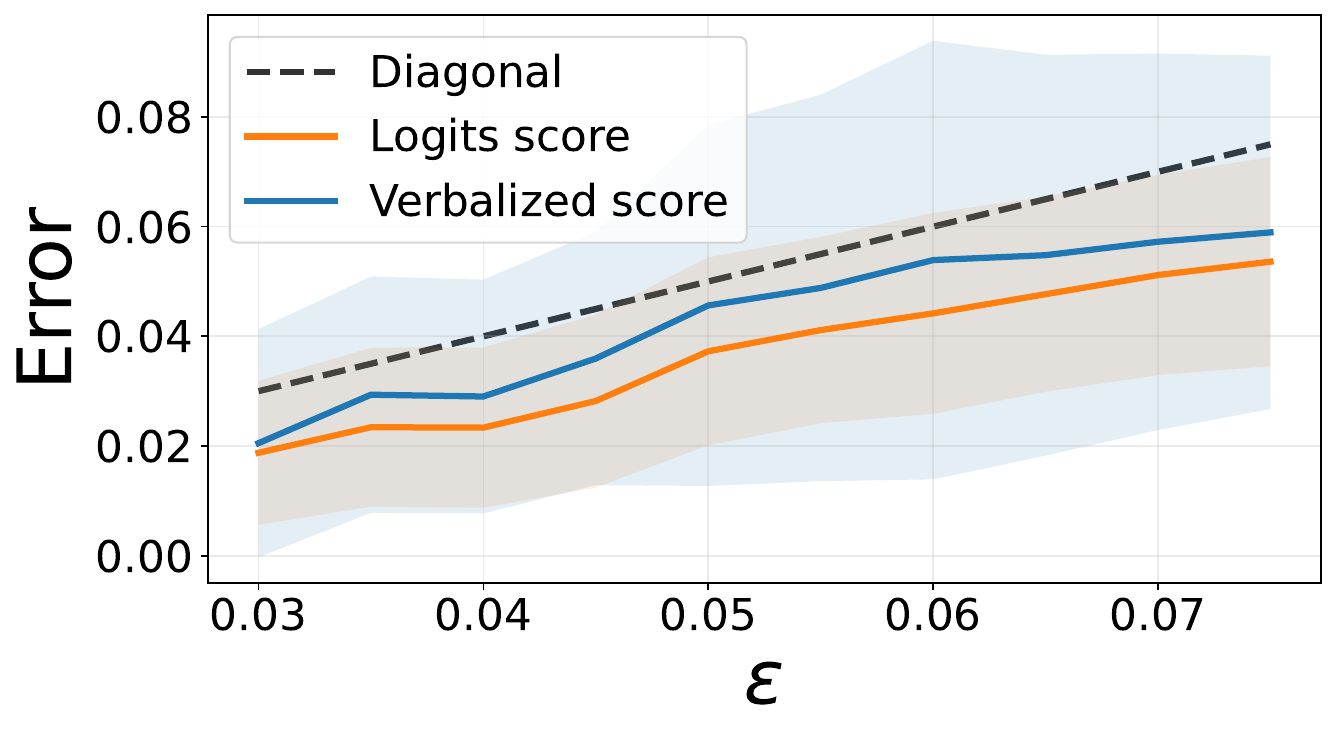}
        \label{fig:pdf_ce}
    \end{subfigure}
    \begin{subfigure}[b]{0.49\textwidth}
        \centering
        \includegraphics[height=5.0cm,width=8cm]{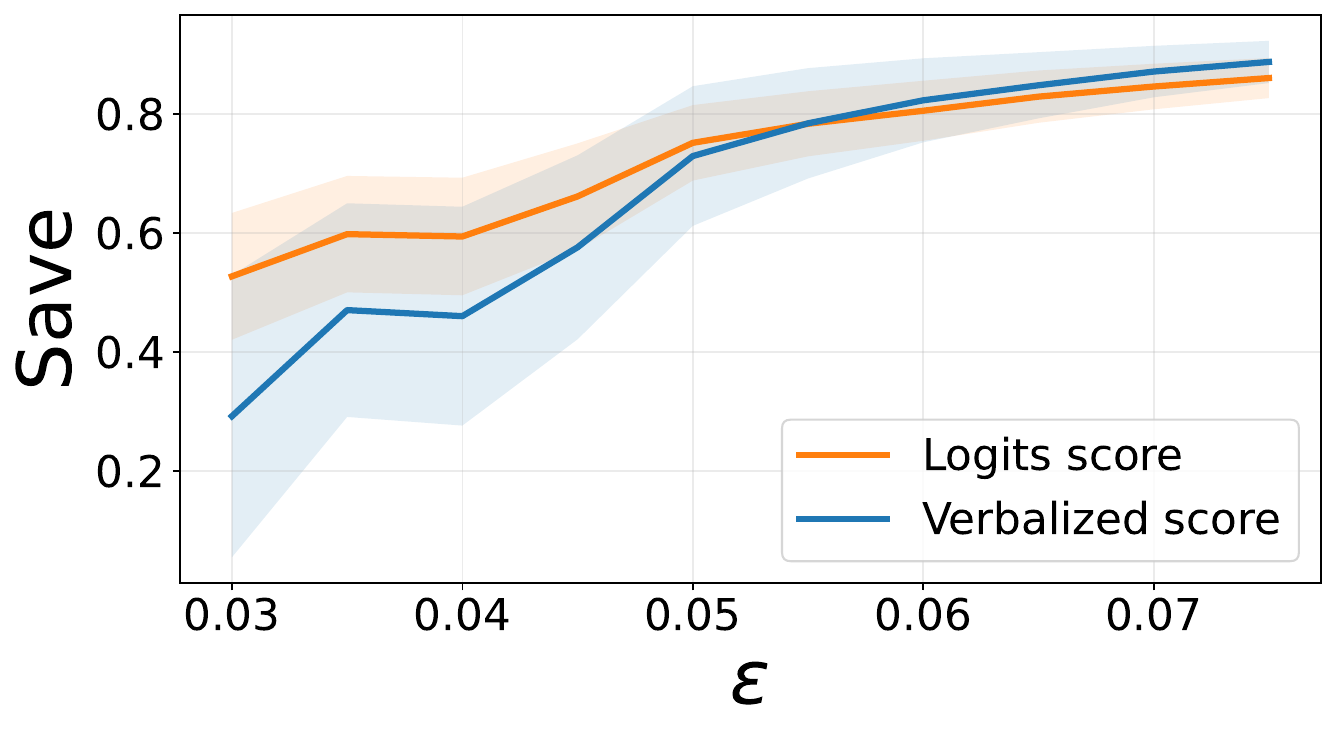}
        \label{fig:pdf_logitnorm}
    \end{subfigure}
\caption{\textbf{The logits-based uncertainty score achieves better cost-error tradeoff than the verbalized score.}
Error and cost savings of HyPAC with different uncertainty score functions under a confidence level of $\alpha = 0.05$. 
Experiments are conducted on MATH-500.
The shaded areas represent standard deviations.}
     \label{fig:verbalized_vs_logits}
    \vspace{-8pt}
\end{figure*}

\paragraph{Datasets.}
We conduct experiments on general tasks, math and text reasoning tasks, and coding tasks, using MMLU-Redux~\citep{gema2025are}, MATH-500~\citep{lightman2023lets}, MATH-L5~\citep{hendrycks2021measuring}, Zebra-Logic~\citep{lin2025zebralogic}, and $\text{HumanEval}^{+}$~\citep{liu2023your} datasets.
Table~\ref{tab:accuracy} reports the accuracy of each LLM on these datasets.
For each dataset, we partition the original test set into a calibration set and a test set randomly.
Table~\ref{tab:datasets_split} provides details on the specific splitting strategies.

\paragraph{Cost functions.}
We consider two cost functions: token-based and API-based.
For the token-based cost function, the cost of annotating a sample is defined as the number of tokens generated by the LLM. 
The expert annotation cost is set to the maximum allowed new tokens, which is 32,768.
For the API cost function, the annotation cost is computed: 
\[
c_{\mathrm{API}}^{\tilde{f}} = p_{\mathrm{in}}^{\tilde{f}} \cdot n_{\mathrm{in}} + p_{\mathrm{out}}^{\tilde{f}} \cdot n_{\mathrm{out}},
\]
where $n_{\mathrm{in}}$ and $n_{\mathrm{out}}$ denote the numbers of input and generated tokens, respectively, and $p_{\mathrm{in}}^{\tilde{f}}$ and $p_{\mathrm{out}}^{\tilde{f}}$ are the model-specific per-token prices.
For expert annotation, we adopt the thinking model’s pricing, and set $n_{\mathrm{out}}$ to 32,768.
The exact per-token prices are provided in Appendix~\ref{sec:implementation}.


\paragraph{Baselines.}
We compare HyPAC against four baselines: (1) PAC labeling~\citep{candes2025probably} with non-thinking models and experts, (2) PAC labeling with thinking models and experts, (3) cascaded selective evaluation (CSE)~\citep{jung2024trust}, (4) CoAnnotating~\citep{li2023coannotating}. 
We provide a detailed introduction to these baselines in Appendix~\ref{sec:baselines}.

\paragraph{Evaluation metrics.}
We evaluate the annotation performance by measuring the following metrics: \textbf{annotation error} defined in Eq.~\eqref{eq:anno_error}, and \textbf{cost savings}, defined as the relative reduction in annotation cost achieved by the hybrid annotation methods compared to expert-only annotation:
\begin{equation*}
\left(
1 - 
\frac{\sum_{i=1}^{n_{\text{test}}} C(\hat f, x_i)}{\sum_{i=1}^{n_{\text{test}}} C(y, x_i)}
\right)
\times 100\%.
\end{equation*}

\paragraph{Implementation Details}
Across all experiments, we fix the confidence level at $\alpha = 0.05$ and the sampling weight at $p = 0.9$, while varying $\epsilon$.
Each experiment is repeated 100 times, and we report the mean annotation error and cost savings.
Unless otherwise specified, we use the CLT-based UCB, the token-based cost function, and the binary 0--1 loss.
More details of implementation are provided in Appendix~\ref{sec:implementation}.

\subsection{Results}

\paragraph{HyPAC tightly controls annotation error and reduces annotation costs compared with other provable methods.}
Table~\ref{tab:method_comparison} compares HyPAC with baseline methods on five annotation tasks.
HyPAC keeps the error below~$\epsilon$ across all benchmarks while saving more costs.
For example, on HumanEval$^{+}$ with $\epsilon = 0.05$ and $\alpha = 0.05$ using logits-based scores, HyPAC has an average error of 3.47\% and cost savings of 86.69\%.
In contrast, PAC labeling baselines save 60.89\% and 72.86\% of costs, respectively.
HyPAC increases cost savings by 13\% over the best baseline, showing that using both models is better for saving costs.
We see similar high savings on other datasets and score functions.
HyPAC provides a balance between cost and accuracy in most cases.
Appendix~\ref{subsec:additional_main_experiment} shows results for two other model pairs.

\paragraph{HyPAC controls the annotation error with different uncertainty scores and is more stable in the logits-based score.}
HyPAC works with any uncertainty score and ensures error control while reducing costs regardless of the score function used.
Figure~\ref{fig:verbalized_vs_logits} compares logits-based and verbalized scores across different error tolerances $\epsilon$. Both scores achieve error control, but they show different levels of stability.
The logits-based score shows smoother and more stable results, with lower standard deviations in both error and cost savings.
Reliability diagrams in Figure~\ref{fig:reliability_8panel} in Appendix~\ref{subsec:more_reliability} show that logits-based scores have lower expected calibration error~\citep{naeini2015obtaining}, suggesting better reliability.
In line with these findings, Table~\ref{tab:method_comparison} shows that the logits-based score saves more costs on most datasets.
These results show that while HyPAC works with various scores, the choice of score impacts performance stability.

\paragraph{HyPAC provides strict error control while CoAnnotating fails.}
We compare HyPAC with heuristic methods like CoAnnotating~\citep{li2023coannotating} across multiple benchmarks.
As shown in Figure~\ref{fig:annotations_compare_CoAnnotating}, these heuristic methods fail to control the annotation error.
Even when these algorithms are calibrated with a target error during their process, they still cannot stay within the error target.
In contrast, HyPAC controls the error below~$\epsilon$ because of the PAC guarantee. 

\paragraph{Additional results.}
Due to space constraints, we place additional analyses in the appendix.
Appendix~\ref{subsec:semantic_loss} presents the performance of HyPAC under the semantic loss function.
Appendix~\ref{subsec:api_cost_performance} reports results under API-based cost functions.
Appendix~\ref{subsec:calibration_set_size} conducts an ablation study on the impact of calibration set size, and shows the robustness of HyPAC on calibration set size.
Appendix~\ref{subsec:different_ucb} includes experiments with distribution-free, finite-sample UCBs, complementing our main results based on the CLT-based asymptotic UCB.
Appendix~\ref{subsection:distribution_shift} shows HyPAC's robustness against distribution shift.

\begin{figure}
    \centering
    \includegraphics[width=0.8\linewidth]{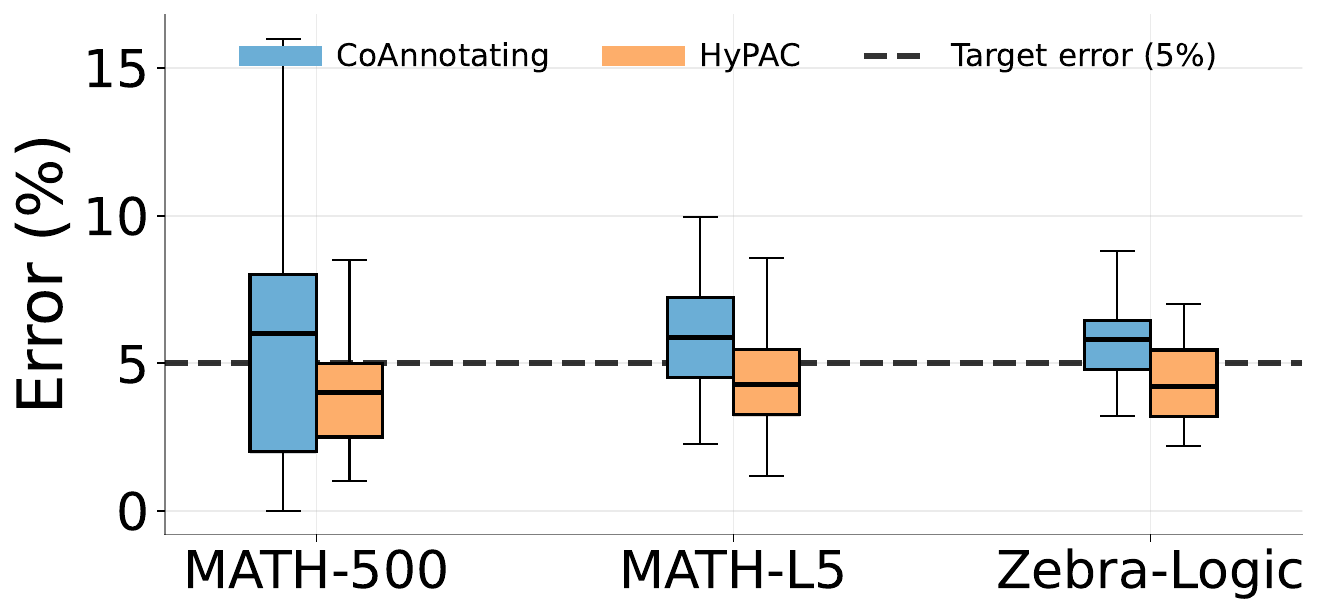}
    \caption{
    \textbf{HyPAC controls the annotations error while CoAnnotating fails.}    
    Comparison of annotation error rates between HyPAC and CoAnnotating~\citep{li2023coannotating} across three benchmarks. HyPAC controls the error rate under the specified tolerance level \(0.05\) while CoAnnotating fails.}
    \label{fig:annotations_compare_CoAnnotating}
\end{figure}

\section{Related work}
\label{sec:related-work}

\paragraph{PAC-style error rate control}
Modern PAC-style error rate control is mainly built on distribution-free risk control.
Conformal prediction provides finite-sample coverage guarantees under exchangeability~\citep{angelopoulos2025conformal}.
This idea has been extended to control the expected loss on prediction sets, enabling risk control for tasks such as multi-label classification and image segmentation~\citep{bates2021distributionfree}.
Risk control has been further reframed as multiple hypothesis testing, allowing simultaneous control of multiple risks, including false discovery rate~\citep{angelopoulos2025learn}.
These methods have been applied to various domains: conformal language modeling for text generation with guaranteed quality~\citep{quach2024conformal}, and automatically adaptive conformal risk control that adjusts to the difficulty of test samples~\citep{blot2025automatically}.
Localized adaptive risk control has been introduced for online calibration with conditional guarantees~\citep{zecchin2024localized}.
Recently, distribution-free risk control has been applied to data labeling, producing labels that are probably approximately correct~\citep{candes2025probably}.
However, this method only considers a single human-machine switch without optimizing the overall cost across multiple annotation sources.
Our work extends this framework to a multi-source setting, optimizing the routing policy to minimize cost while controlling the labeling error.

\paragraph{Hybrid annotation}
Hybrid annotation combines automatic predictors with human-in-the-loop querying to reduce labeling costs.
Recent advances use large language models as annotators, with studies showing that GPT-3 can serve as a data annotator across multiple NLP tasks~\citep{ding2023gpt3} and that LLMs can match crowdsourced annotators with proper prompting~\citep{he2024annollm}.
Several works explore human-LLM collaboration for annotation, including uncertainty-guided work allocation~\citep{li2023coannotating}, verifier-based quality assessment that routes low-confidence samples to humans~\citep{wang2024humanllm}, unsupervised collaboration between LLMs and small models~\citep{xiao2023freeal}, and multi-agent systems that manage multiple annotation sources~\citep{qin2025crowdagent}.
However, these methods optimize label efficiency through empirical heuristics without providing distribution-free guarantees on the error rate of automatically assigned labels.
Recent work has taken steps toward statistical validity by combining LLM confidence with human annotations~\citep{gligoric2024can}, though this focuses on downstream inference rather than label quality control.
Concurrent work proposes PAC labels that guarantee low overall labeling error with high probability~\citep{candes2025probably}, but uses a fixed confidence threshold without adaptive routing.
Our work extends this direction by integrating distribution-free risk control into a hybrid annotation with multiple annotation sources, achieving label efficiency explicit finite-sample error bounds.

\section{Conclusion}
\label{sec:conclusion}

We presented HyPAC, a hybrid annotation method that labels inputs among multiple annotation sources with different cost-quality trade-offs.
HyPAC calibrates two thresholds using upper confidence bounds, providing a PAC guarantee that the error rate stays below a user-specified tolerance with high probability.
We proved that the selected thresholds achieve the minimum expected cost among all feasible solutions.
Experiments on benchmarks show that HyPAC achieves large cost savings while maintaining the target error rate.
By providing a practical and theoretically grounded approach to hybrid annotation, we hope this work will help researchers and practitioners build more cost-efficient annotation systems while maintaining quality guarantees.


\newpage

\section*{Impact Statement}
This paper presents work whose goal is to advance the field of Machine
Learning. There are many potential societal consequences of our work, none
which we feel must be specifically highlighted here.

\bibliographystyle{icml2026}
\bibliography{main}

\newpage
\appendix
\onecolumn

\section{Proofs of theoretical results}
\label{app:proofs}

This appendix contains the complete proofs of all theoretical results presented in Section~\ref{sec:theory}.

\subsection{Proof of Lemma~\ref{lem:monotonicity}}
\label{app:monotonicity}

In this section, we provide the complete proof of Lemma~\ref{lem:monotonicity}, which establishes the monotonicity of the risk function with respect to both threshold parameters.
This result is fundamental to our theoretical framework.

\begin{proof}[Proof of Lemma~\ref{lem:monotonicity}]
Recall that the risk function is:
\begin{equation*}
R(u_1, u_2) = \mathbb{E}_{(x,y) \sim \mathcal{D}}[\ell(y, T_{u_1,u_2}(x))]
\end{equation*}
where $T_{u_1,u_2}(x) = \tilde{f}_1(x)$ if $U(x) \le u_1$, $T_{u_1,u_2}(x) = \tilde{f}_2(x)$ if $u_1 < U(x) \le u_2$, and $T_{u_1,u_2}(x) = y$ otherwise.

\textbf{Monotonicity in $u_1$.}
Fix $u_2$ and consider $u_1 < u_1' \le u_2$.
The mappings $T_{u_1,u_2}$ and $T_{u_1',u_2}$ differ only for examples with $u_1 < U(x) \le u_1'$: these are routed to $\tilde{f}_2$ under $T_{u_1,u_2}$ but to $\tilde{f}_1$ under $T_{u_1',u_2}$.
The risk difference is:
\begin{equation*}
R(u_1', u_2) - R(u_1, u_2) = \mathbb{P}(u_1 < U(x) \le u_1') \cdot \mathbb{E}_{u_1 < U(x) \le u_1'}[\ell(y, \tilde{f}_1(x)) - \ell(y, \tilde{f}_2(x))]
\end{equation*}
Since $\tilde{f}_2$ is more accurate than $\tilde{f}_1$, we have $\mathbb{E}[\ell(y, \tilde{f}_1(x))] \ge \mathbb{E}[\ell(y, \tilde{f}_2(x))]$ for any conditional distribution.
Thus $R(u_1', u_2) - R(u_1, u_2) \ge 0$.

\textbf{Monotonicity in $u_2$.}
Fix $u_1$ and consider $u_2 < u_2'$ with $u_1 \le u_2$.
The mappings differ only for examples with $u_2 < U(x) \le u_2'$: these are routed to human annotation under $T_{u_1,u_2}$ but to $\tilde{f}_2$ under $T_{u_1,u_2'}$.
Since $\ell(y, y) = 0$, the risk difference is:
\begin{equation*}
R(u_1, u_2') - R(u_1, u_2) = \mathbb{P}(u_2 < U(x) \le u_2') \cdot \mathbb{E}_{u_2 < U(x) \le u_2'}[\ell(y, \tilde{f}_2(x))] \ge 0
\end{equation*}
where the inequality follows from the non-negativity of the loss function.
\end{proof}

\subsection{Proof of Lemma~\ref{lem:cost-monotonicity}}
\label{app:cost-monotonicity}

\begin{proof}
We prove both the population version and the empirical version.

\textbf{Population version.}
The expected cost can be written as:
\begin{equation*}
\mathbb{E}_{x \sim \mathcal{D}_x}[C(u_1,u_2,x)] = \mathbb{E}[c_1(x) \cdot \mathbf{1}\{U(x) \le u_1\}] + \mathbb{E}[c_2(x) \cdot \mathbf{1}\{u_1 < U(x) \le u_2\}] + \mathbb{E}[c_h(x) \cdot \mathbf{1}\{U(x) > u_2\}]
\end{equation*}
Rearranging terms:
\begin{equation*}
\mathbb{E}_{x \sim \mathcal{D}_x}[C(u_1,u_2,x)] = \mathbb{E}[c_h(x)] + \mathbb{E}[(c_1(x) - c_2(x)) \cdot \mathbf{1}\{U(x) \le u_1\}] + \mathbb{E}[(c_2(x) - c_h(x)) \cdot \mathbf{1}\{U(x) \le u_2\}]
\end{equation*}

Under the cost ordering assumption $c_1(x) \le c_2(x) \le c_h(x)$, we have:
\begin{equation*}
c_1(x) - c_2(x) \le 0 \quad \text{and} \quad c_2(x) - c_h(x) \le 0 \quad \text{for almost all } x
\end{equation*}

For any $u_1 < u_1'$ with $u_1' \le u_2$, the cost difference is:
\begin{align*}
\mathbb{E}[C(u_1',u_2,x)] - \mathbb{E}[C(u_1,u_2,x)] &= \mathbb{E}[(c_1(x) - c_2(x)) \cdot (\mathbf{1}\{U(x) \le u_1'\} - \mathbf{1}\{U(x) \le u_1\})] \\
&= \mathbb{E}[(c_1(x) - c_2(x)) \cdot \mathbf{1}\{u_1 < U(x) \le u_1'\}] \le 0
\end{align*}
where the inequality follows from $c_1(x) - c_2(x) \le 0$.

Similarly, for any $u_2 < u_2'$, we have:
\begin{align*}
\mathbb{E}[C(u_1,u_2',x)] - \mathbb{E}[C(u_1,u_2,x)] &= \mathbb{E}[(c_2(x) - c_h(x)) \cdot \mathbf{1}\{u_2 < U(x) \le u_2'\}] \le 0
\end{align*}

Therefore, the expected cost is non-increasing in both $u_1$ and $u_2$.

\textbf{Empirical version.}
For any finite sample $\{x_i\}_{i=1}^m$, the empirical cost is:
\begin{equation*}
\hat{C}(u_1, u_2) = \frac{1}{m}\sum_{i=1}^m C(u_1, u_2, x_i)
\end{equation*}
where $C(u_1, u_2, x_i)$ is the cost for example $x_i$ under the routing policy $T_{u_1,u_2}$.

By the definition of the routing policy, we have:
\begin{equation*}
C(u_1, u_2, x_i) = \begin{cases}
c_1(x_i) & \text{if } U(x_i) \le u_1 \\
c_2(x_i) & \text{if } u_1 < U(x_i) \le u_2 \\
c_h(x_i) & \text{if } U(x_i) > u_2
\end{cases}
\end{equation*}

For any $u_1 < u_1'$ with $u_1' \le u_2$, the empirical cost difference is:
\begin{align*}
\hat{C}(u_1', u_2) - \hat{C}(u_1, u_2) &= \frac{1}{m}\sum_{i=1}^m [C(u_1', u_2, x_i) - C(u_1, u_2, x_i)] \\
&= \frac{1}{m}\sum_{i: u_1 < U(x_i) \le u_1'} [c_1(x_i) - c_2(x_i)] \le 0
\end{align*}
where the inequality follows from $c_1(x_i) \le c_2(x_i)$ for all $i$.

Similarly, for any $u_2 < u_2'$:
\begin{align*}
\hat{C}(u_1, u_2') - \hat{C}(u_1, u_2) &= \frac{1}{m}\sum_{i: u_2 < U(x_i) \le u_2'} [c_2(x_i) - c_h(x_i)] \le 0
\end{align*}

Therefore, the empirical cost is non-increasing in both $u_1$ and $u_2$.

\textbf{Strict monotonicity.}
When the cost ordering is strict (i.e., $c_1(x) < c_2(x) < c_h(x)$ for almost all $x$), the inequalities in both the population and empirical versions become strict whenever the corresponding indicator sets have positive measure (for population) or are non-empty (for empirical).
\end{proof}

\subsection{Proof of Theorem~\ref{thm:main}}
\label{app:main-theorem}

In this section, we provide the complete proof of Theorem~\ref{thm:main}, which establishes that HyPAC provides a PAC guarantee for error control.
We use the fixed-sequence testing framework to handle the fact that the threshold pair is selected based on the calibration data.
The proof consists of three main steps.
First, we show that the UCB provides valid coverage for any fixed threshold pair (Eq.~\eqref{eq:ucb-coverage-fixed}).
Second, we extend this coverage to the data-dependent selected pair: the key observation is that $\{R(\hat{u}_1, \hat{u}_2) > \epsilon\} \subseteq \{R(\hat{u}_1, \hat{u}_2) > \hat{L}_{\hat{u}_1,\hat{u}_2}(\alpha)\}$, since the algorithm only selects pairs with $\hat{L}_{\hat{u}_1,\hat{u}_2}(\alpha) \le \epsilon$.
Third, we use the monotonicity of the risk function (Lemma~\ref{lem:monotonicity}) to apply fixed-sequence testing without Bonferroni correction.
The monotonicity ensures that if $R(u_1, u_2) \le \epsilon$ for some pair, then $R(u_1', u_2') \le \epsilon$ for all $u_1' \le u_1$ and $u_2' \le u_2$.
This nested structure allows us to test threshold pairs in a fixed sequence (from largest to smallest) and stop at the first rejection, maintaining the significance level $\alpha$ without multiple testing correction.

\begin{proof}[Proof of Theorem~\ref{thm:main}]


By Assumption~\ref{assump:ucb}, for any fixed threshold pair $(u_1, u_2)$ with $0 \le u_1 \le u_2 \le 1$, the upper confidence bound $\hat{L}_{u_1,u_2}(\alpha)$ satisfies:
\begin{equation}
\label{eq:ucb-coverage-fixed}
\mathbb{P}_{\mathcal{S}_{\text{cal}}} \left( R(u_1, u_2) \le \hat{L}_{u_1,u_2}(\alpha) \right) \ge 1 - \alpha
\end{equation}

This means that with probability at least $1 - \alpha$ over the random draw of the calibration set, the UCB is an upper bound on the true risk.
The challenge is that the threshold pair $(\hat{u}_1, \hat{u}_2)$ selected by Algorithm~\ref{alg:hypac} is not fixed but depends on the calibration data $\mathcal{S}_{\text{cal}}$.
Specifically, the algorithm selects the maximal feasible pair in the \textit{lexicographic} order $(u_2, u_1)$:
\begin{equation*}
(\hat{u}_1, \hat{u}_2) = \arg\max_{(u_1, u_2) \in \mathcal{F}} (u_2, u_1)
\end{equation*}
where $\mathcal{F} = \{(u_1, u_2): \hat{L}_{u_1,u_2}(\alpha) \le \epsilon, 0 \le u_1 \le u_2 \le 1\}$ is the feasible set of threshold pairs whose UCB is below $\epsilon$.
The lexicographic maximum first maximizes $u_2$, then maximizes $u_1$ subject to $u_1 \le u_2$.

The key insight is that the event $\{R(\hat{u}_1, \hat{u}_2) > \epsilon\}$ implies the event $\{R(\hat{u}_1, \hat{u}_2) > \hat{L}_{\hat{u}_1,\hat{u}_2}(\alpha)\}$ because the algorithm only selects pairs with $\hat{L}_{\hat{u}_1,\hat{u}_2}(\alpha) \le \epsilon$.
Formally, suppose $R(\hat{u}_1, \hat{u}_2) > \epsilon$.
Since Algorithm~\ref{alg:hypac} only selects pairs from the feasible set $\mathcal{F}$, we have $\hat{L}_{\hat{u}_1,\hat{u}_2}(\alpha) \le \epsilon$.
Therefore:
\begin{equation*}
R(\hat{u}_1, \hat{u}_2) > \epsilon \ge \hat{L}_{\hat{u}_1,\hat{u}_2}(\alpha)
\end{equation*}

This means:
\begin{equation*}
\{R(\hat{u}_1, \hat{u}_2) > \epsilon\} \subseteq \{R(\hat{u}_1, \hat{u}_2) > \hat{L}_{\hat{u}_1,\hat{u}_2}(\alpha)\}
\end{equation*}

Then
\begin{align*}
\mathbb{P}_{\mathcal{S}_{\text{cal}}} \left( R(\hat{u}_1, \hat{u}_2) > \epsilon \right) &\le \mathbb{P}_{\mathcal{S}_{\text{cal}}} \left( R(\hat{u}_1, \hat{u}_2) > \hat{L}_{\hat{u}_1,\hat{u}_2}(\alpha) \right).
\end{align*}

Now, the event $\{R(\hat{u}_1, \hat{u}_2) > \hat{L}_{\hat{u}_1,\hat{u}_2}(\alpha)\}$ means that the UCB fails to upper bound the true risk for the selected pair.
By the monotonicity of the risk function (Lemma~\ref{lem:monotonicity}), the null hypotheses $H_{u_1,u_2}: R(u_1, u_2) > \epsilon$ have a nested structure: if $H_{u_1,u_2}$ is true, then $H_{u_1',u_2'}$ is also true for all $u_1' \ge u_1$ and $u_2' \ge u_2$.
This enables fixed-sequence testing~\citep{angelopoulos2025learn}.

The key property of fixed-sequence testing is that it controls the family-wise error rate at level $\alpha$ without any correction.
Formally, let $(u_1^{(1)}, u_2^{(1)}), (u_1^{(2)}, u_2^{(2)}), \ldots$ be the sequence of threshold pairs tested by the algorithm (ordered from largest to smallest).
Let $k^*$ be the index of the first pair whose null hypothesis is true (i.e., $R(u_1^{(k^*)}, u_2^{(k^*)}) > \epsilon$).
Then:
\begin{align*}
&\mathbb{P}_{\mathcal{S}_{\text{cal}}}(\text{reject } H_{u_1^{(k^*)},u_2^{(k^*)}}) \\
&= \mathbb{P}_{\mathcal{S}_{\text{cal}}}(\hat{L}_{u_1^{(k^*)},u_2^{(k^*)}}(\alpha) \le \epsilon \mid R(u_1^{(k^*)}, u_2^{(k^*)}) > \epsilon) \\
&\le \alpha
\end{align*}

Since the algorithm stops at the first rejection, it can only make an error if it rejects $H_{u_1^{(k^*)},u_2^{(k^*)}}$.
Therefore:
\begin{equation*}
\mathbb{P}_{\mathcal{S}_{\text{cal}}}(R(\hat{u}_1, \hat{u}_2) > \epsilon) \le \alpha
\end{equation*}

Finally, we verify that Algorithm~\ref{alg:hypac} indeed only selects threshold pairs whose UCB is below $\epsilon$.
By construction, the algorithm searches over the grid $\mathcal{G}$ and selects:
\begin{equation*}
\hat{u}_2 = \max\{u_2 : \exists u_1 \le u_2 \text{ s.t. } \hat{L}_{u_1,u_2}(\alpha) \le \epsilon\}
\end{equation*}
\begin{equation*}
\hat{u}_1 = \max\{u_1 : u_1 \le \hat{u}_2 \text{ and } \hat{L}_{u_1,\hat{u}_2}(\alpha) \le \epsilon\}
\end{equation*}
This ensures that $\hat{L}_{\hat{u}_1,\hat{u}_2}(\alpha) \le \epsilon$ by construction.
If no such pair exists (i.e., $\mathcal{F} = \emptyset$), the algorithm can return $(0, 0)$, which routes all examples to the thinking model and human annotation, guaranteeing zero error.

Combining the three steps, we have shown that:
\begin{equation*}
\mathbb{P}_{\mathcal{S}_{\text{cal}}} \left( R(\hat{u}_1, \hat{u}_2) \le \epsilon \right) \ge 1 - \alpha
\end{equation*}

This completes the proof of Theorem~\ref{thm:main}.
\end{proof}

\begin{remark}[Role of monotonicity in avoiding multiple testing correction]
A naive approach would apply a union bound over all threshold pairs in the grid $\mathcal{G}$, requiring a Bonferroni correction that sets the significance level for each UCB to $\alpha / |\mathcal{G}|$.
This would make the UCBs wider and the selected thresholds more conservative.
However, the monotonicity of the risk function (Lemma~\ref{lem:monotonicity}) allows us to avoid this correction entirely.
The nested structure of the null hypotheses enables fixed-sequence testing: the algorithm tests threshold pairs from largest to smallest, and stops at the first rejection.
By the monotonicity, all pairs with smaller thresholds also satisfy $R \le \epsilon$, so we do not need to test them.
This is why the PAC guarantee holds with significance level $\alpha$, not $\alpha / |\mathcal{G}|$. 
\end{remark}

\subsection{Proof of Theorem~\ref{thm:cost-optimal}}
\label{app:cost-optimal}

\begin{proof}
Fix a calibration set $\mathcal{S}_{\text{cal}}$.
By Lemma~\ref{lem:cost-monotonicity}, both $\hat{C}(u_1, u_2)$ and $\mathbb{E}_{x \sim \mathcal{D}_x}[C(u_1, u_2, x)]$ are non-increasing in both $u_1$ and $u_2$.
Algorithm~\ref{alg:hypac} computes:
\[
(\hat{u}_1, \hat{u}_2) = \arg\min_{(u_1, u_2) \in \mathcal{F}_{\text{PAC}}(\mathcal{S}_{\text{cal}})} \hat{C}(u_1, u_2),
\]
where $\mathcal{F}_{\text{PAC}}(\mathcal{S}_{\text{cal}}) = \{(u_1, u_2) : \hat{L}_{u_1,u_2}(\alpha) \le \epsilon, \; 0 \le u_1 \le u_2 \le 1\}$.
Since $\hat{C}(u_1, u_2)$ is non-increasing in both arguments, the minimum over $\mathcal{F}_{\text{PAC}}(\mathcal{S}_{\text{cal}})$ is achieved at the maximum feasible thresholds.
By the same monotonicity, $\mathbb{E}_{x \sim \mathcal{D}_x}[C(u_1, u_2, x)]$ is also minimized at the maximum feasible thresholds.
Therefore,
\begin{equation*}
\mathbb{E}_{x \sim \mathcal{D}_x}[C(\hat{u}_1, \hat{u}_2, x)] = \min_{(u_1, u_2) \in \mathcal{F}_{\text{PAC}}(\mathcal{S}_{\text{cal}})} \mathbb{E}_{x \sim \mathcal{D}_x}[C(u_1, u_2, x)].
\end{equation*}
Note that $\mathcal{F}_{\text{PAC}}(\mathcal{S}_{\text{cal}})$ depends on $\mathcal{S}_{\text{cal}}$ via empirical bounds, but $\mathbb{E}_{x \sim \mathcal{D}_x}[C(u_1, u_2, x)]$ is a population quantity independent of $\mathcal{S}_{\text{cal}}$.
\end{proof}

\subsection{Finite-sample guarantee}
\label{app:empirical-risk}

While Theorem~\ref{thm:main} provides a guarantee on the true risk $R(\hat{u}_1, \hat{u}_2)$, in practice we often want to evaluate the performance on a finite test set.
We now provide a bound on the empirical risk that accounts for the finite test set size.

\begin{theorem}[Empirical risk PAC guarantee]
\label{thm:empirical}
Let $(\hat{u}_1, \hat{u}_2)$ be the threshold pair selected by Algorithm~\ref{alg:hypac} with error tolerance $\epsilon$ and significance level $\alpha$.
Let $\mathcal{S}_{\text{test}} = \{x_j\}_{j=1}^N$ be a test set drawn independently from the same distribution as the calibration set.
Define the empirical risk as:
\begin{equation*}
\hat{R}(\hat{u}_1, \hat{u}_2) = \frac{1}{N} \sum_{j=1}^N \ell(y_j, T_{\hat{u}_1,\hat{u}_2}(x_j))
\end{equation*}
Assuming the loss function is bounded in $[0, B]$, for any $t > 0$:
\begin{equation*}
\mathbb{P}_{\mathcal{S}_{\text{cal}}, \mathcal{S}_{\text{test}}} \left( \hat{R}(\hat{u}_1, \hat{u}_2) \le \epsilon + t \right) \ge 1 - \alpha - (1-\alpha)\exp\left(-\frac{2Nt^2}{B^2}\right)
\end{equation*}
where the probability is taken over both the calibration set and the test set.
\end{theorem}

This theorem tells us how to interpret the empirical error rate on a finite test set: if we observe $\hat{R}(\hat{u}_1, \hat{u}_2)$, the true error rate is close to this value, with deviation controlled by $N$.

\begin{proof}[Proof of Theorem~\ref{thm:empirical}]

The proof proceeds as follows.
We first establish a concentration bound for the empirical risk around the true risk.
We then combine this with the PAC guarantee from Theorem~\ref{thm:main} using a union bound argument.

Fix a threshold pair $(u_1, u_2)$ and consider the empirical risk:
\begin{equation*}
\hat{R}(u_1, u_2) = \frac{1}{N} \sum_{j=1}^N \ell(y_j, T_{u_1,u_2}(x_j))
\end{equation*}

Since the test examples are drawn independently from $\mathcal{D}$, the losses are i.i.d.\ random variables with:
\begin{equation*}
\mathbb{E}_{x_j \sim \mathcal{D}}[\ell(y_j, T_{u_1,u_2}(x_j))] = R(u_1, u_2)
\end{equation*}

Since the loss is bounded in $[0, B]$, Hoeffding's inequality gives:
\begin{equation*}
\mathbb{P}\left( \frac{1}{N} \sum_{i=1}^N X_i - \mathbb{E}\left[\frac{1}{N} \sum_{i=1}^N X_i\right] \ge t \right) \le \exp\left( -\frac{2N^2 t^2}{\sum_{i=1}^N (b_i - a_i)^2} \right)
\end{equation*}

With $b_j - a_j = B$ for all $j$, we have $\sum_{j=1}^N (b_j - a_j)^2 = NB^2$, so:
\begin{equation*}
\mathbb{P}_{\mathcal{S}_{\text{test}}}\left( \hat{R}(u_1, u_2) - R(u_1, u_2) > t \right) \le \exp\left( -\frac{2Nt^2}{B^2} \right)
\end{equation*}

Now we combine this with Theorem~\ref{thm:main}.
We decompose the probability by conditioning on whether $R(\hat{u}_1, \hat{u}_2) \le \epsilon$:
\begin{align*}
&\mathbb{P}_{\mathcal{S}_{\text{cal}}, \mathcal{S}_{\text{test}}}\left( \hat{R}(\hat{u}_1, \hat{u}_2) \le \epsilon + t \right) \\
&= \mathbb{P}_{\mathcal{S}_{\text{cal}}, \mathcal{S}_{\text{test}}}\left( \hat{R} \le \epsilon + t \mid R \le \epsilon \right) \cdot \mathbb{P}_{\mathcal{S}_{\text{cal}}}\left( R \le \epsilon \right) \\
&\quad + \mathbb{P}_{\mathcal{S}_{\text{cal}}, \mathcal{S}_{\text{test}}}\left( \hat{R} \le \epsilon + t \mid R > \epsilon \right) \cdot \mathbb{P}_{\mathcal{S}_{\text{cal}}}\left( R > \epsilon \right)
\end{align*}
where we write $R = R(\hat{u}_1, \hat{u}_2)$ and $\hat{R} = \hat{R}(\hat{u}_1, \hat{u}_2)$ for brevity.

Since the second term is non-negative, we have:
\begin{equation*}
\mathbb{P}_{\mathcal{S}_{\text{cal}}, \mathcal{S}_{\text{test}}}\left( \hat{R} \le \epsilon + t \right) \ge \mathbb{P}_{\mathcal{S}_{\text{cal}}, \mathcal{S}_{\text{test}}}\left( \hat{R} \le \epsilon + t \mid R \le \epsilon \right) \cdot \mathbb{P}_{\mathcal{S}_{\text{cal}}}\left( R \le \epsilon \right)
\end{equation*}

When $R \le \epsilon$, the event $\{\hat{R} \le \epsilon + t\}$ is implied by $\{\hat{R} - R \le t\}$.
By Hoeffding's inequality:
\begin{equation*}
\mathbb{P}_{\mathcal{S}_{\text{test}}}\left( \hat{R} \le \epsilon + t \mid R \le \epsilon \right) \ge \mathbb{P}_{\mathcal{S}_{\text{test}}}\left( \hat{R} - R \le t \right) \ge 1 - \exp\left(-\frac{2Nt^2}{B^2}\right)
\end{equation*}

By Theorem~\ref{thm:main}, $\mathbb{P}_{\mathcal{S}_{\text{cal}}}(R \le \epsilon) \ge 1 - \alpha$.

Combining these:
\begin{align*}
\mathbb{P}_{\mathcal{S}_{\text{cal}}, \mathcal{S}_{\text{test}}}\left( \hat{R} \le \epsilon + t \right) &\ge \left(1 - \exp\left(-\frac{2Nt^2}{B^2}\right)\right) \cdot (1 - \alpha) \\
&= 1 - \alpha - \exp\left(-\frac{2Nt^2}{B^2}\right) + \alpha \exp\left(-\frac{2Nt^2}{B^2}\right) \\
&= 1 - \alpha - (1 - \alpha)\exp\left(-\frac{2Nt^2}{B^2}\right)
\end{align*}

This completes the proof of Theorem~\ref{thm:empirical}.
\end{proof}

\section{UCB validity examples}
\label{app:ucb-validity}

In this section, we provide concrete examples of how to construct valid upper confidence bounds (UCBs) that satisfy Assumption~\ref{assump:ucb}.
The UCB must satisfy: for any fixed threshold pair $(u_1, u_2)$,
\begin{equation*}
\mathbb{P}_{\mathcal{S}_{\text{cal}}} \left( R(u_1, u_2) \le \hat{L}_{u_1,u_2}(\alpha) \right) \ge 1 - \alpha
\end{equation*}

\textbf{CLT-based UCB.}
If we use the CLT-based UCB from Eq.~\eqref{eq:ucb}:
\begin{equation*}
\hat{L}_{u_1,u_2}(\alpha) = \hat{R}_{\text{IS}}(u_1, u_2) + z_{1-\alpha} \frac{\hat{\sigma}_W(u_1, u_2)}{\sqrt{m}}
\end{equation*}
where $\hat{R}_{\text{IS}}(u_1, u_2)$ is the importance-weighted risk estimate and $\hat{\sigma}_W(u_1, u_2)$ is the sample standard deviation of the importance-weighted losses.
By the central limit theorem, for large enough $m$:
\begin{equation*}
\frac{\hat{R}_{\text{IS}}(u_1, u_2) - R(u_1, u_2)}{\hat{\sigma}_W(u_1, u_2) / \sqrt{m}} \xrightarrow{d} \mathcal{N}(0, 1)
\end{equation*}
This implies that:
\begin{equation*}
\mathbb{P}\left( R(u_1, u_2) \le \hat{R}_{\text{IS}}(u_1, u_2) + z_{1-\alpha} \frac{\hat{\sigma}_W(u_1, u_2)}{\sqrt{m}} \right) =  1 - \alpha + o(1). 
\end{equation*}

\textbf{Hoeffding-based UCB.}
For finite samples, we can use Hoeffding's inequality to obtain a non-asymptotic bound.
Since the loss is bounded, the importance-weighted losses $W_i = \frac{Z_i}{p} \ell(y_i, T_{u_1,u_2}(x_i))$ are bounded in $[0, B/p]$.
By Hoeffding's inequality:
\begin{equation*}
\mathbb{P}\left( \hat{R}_{\text{IS}}(u_1, u_2) - R(u_1, u_2) \ge t \right) \le \exp\left( -\frac{2m t^2}{(B/p)^2} \right)
\end{equation*}
Setting the right-hand side equal to $\alpha$ and solving for $t$:
\begin{equation*}
t = \frac{B}{p\sqrt{2m}} \sqrt{\log(1/\alpha)}
\end{equation*}
This gives us a Hoeffding-based UCB:
\begin{equation*}
\hat{L}_{u_1,u_2}^{\text{Hoeff}}(\alpha) = \hat{R}_{\text{IS}}(u_1, u_2) + \frac{B}{p\sqrt{2m}} \sqrt{\log(1/\alpha)}
\end{equation*}
which satisfies:
\begin{equation*}
\mathbb{P}_{\mathcal{S}_{\text{cal}}} \left( R(u_1, u_2) \le \hat{L}_{u_1,u_2}^{\text{Hoeff}}(\alpha) \right) \ge 1 - \alpha
\end{equation*}

\textbf{Empirical Bernstein bound.}
A tighter bound can be obtained using the empirical Bernstein inequality, which adapts to the variance of the data.
Let $\hat{V} = \frac{1}{m-1} \sum_{i=1}^m (W_i - \hat{R}_{\text{IS}})^2$ be the sample variance.
The empirical Bernstein UCB is:
\begin{equation*}
\hat{L}_{u_1,u_2}^{\text{Bern}}(\alpha) = \hat{R}_{\text{IS}}(u_1, u_2) + \sqrt{\frac{2\hat{V} \log(2/\alpha)}{m}} + \frac{7B \log(2/\alpha)}{3(m-1)p}
\end{equation*}
This bound is tighter than Hoeffding when the variance is small.

\textbf{Betting-based bound.}
We also consider a betting-based upper confidence bound adapted from the finite-sample confidence sequence of \citet{waudby-smith2024estimating}.
Let
\[
W_i = \frac{Z_i}{p}\,\ell\!\left(y_i, T_{u_1,u_2}(x_i)\right),
\qquad
0 \le W_i \le \frac{1}{p},
\]
and define the normalized variables $X_i = p W_i \in [0,1]$.
Let $\mu = \mathbb{E}[X_i] = p\,R(u_1,u_2)$.

For $t=1,\dots,m$, define
\[
\hat{\mu}_t
= \frac{0.5 + \sum_{j=1}^t X_j}{t+1},
\qquad
\hat{\sigma}_t^2
= \frac{0.25 + \sum_{j=1}^t (X_j - \hat{\mu}_t)^2}{t+1},
\qquad
\lambda_t
= \sqrt{\frac{2\log(2/\alpha)}{m\,\hat{\sigma}_{t-1}^2}}.
\]

For any candidate mean $m \in [0,1]$, define
\[
M_t(m)
=
\frac{1}{2}
\max\!\left\{
\prod_{j=1}^t
\left(1 + \min\!\left(\lambda_j, \frac{0.5}{m}\right)(X_j - m)\right),
\;
\prod_{j=1}^t
\left(1 - \min\!\left(\lambda_j, \frac{0.5}{1-m}\right)(X_j - m)\right)
\right\}.
\]

Define the confidence set
\[
\mathcal{C}
=
\bigcap_{t=1}^m
\left\{
m \in [0,1] : M_t(m) < \frac{1}{\alpha}
\right\}.
\]

Let \[\mathcal{C}_R := \{c/p : c\in\mathcal{C}\}.\]
Then,
\[
\mathbb{P}\!\left( R(u_1,u_2) \in \mathcal{C}_R \right) \ge 1-\alpha.
\]

Intuitively, the supermartingale $M_t(m)$ should be thought of as the amount of evidence against $m$ being the true mean. That is, $M_t(m)$ being big suggests that $m$ is unlikely to be the true mean, so the final confidence set is the collection of all $m$ for which the amount of such evidence is small.

\section{Extension to more than three annotation sources}
\label{sec:extension}

Our framework naturally extends to systems with more than three annotation sources.
Suppose we have $K$ annotation sources $\tilde{f}_1, \ldots, \tilde{f}_{K-1}, y$ with cost functions $c_1(x), c_2(x), \ldots, c_{K-1}(x), c_K(x)$ satisfying $c_1(x) \le c_2(x) \le \cdots \le c_K(x)$ for almost all $x$, and accuracies that increase with cost.

We use $K-1$ thresholds $u_1 \le u_2 \le \cdots \le u_{K-1}$ to partition the input space into $K$ regions.
The deployment mapping becomes:
\begin{equation*}
T_{u_1,\ldots,u_{K-1}}(x) = \begin{cases}
\tilde{f}_1(x) & \text{if } U(x) \le u_1 \\
\tilde{f}_2(x) & \text{if } u_1 < U(x) \le u_2 \\
\vdots \\
\tilde{f}_{K-1}(x) & \text{if } u_{K-2} < U(x) \le u_{K-1} \\
y & \text{if } U(x) > u_{K-1}
\end{cases}
\end{equation*}

The risk function is:
\begin{equation*}
R(u_1, \ldots, u_{K-1}) = \mathbb{E}_{(x,y) \sim \mathcal{D}}[\ell(y, T_{u_1,\ldots,u_{K-1}}(x))]
\end{equation*}

The monotonicity result extends naturally: $R$ is monotone non-decreasing in each threshold $u_k$, since increasing any threshold routes more examples to cheaper (less accurate) sources.
The calibration procedure also extends: we construct UCBs for all threshold tuples $(u_1, \ldots, u_{K-1})$ in a discretized grid, then select the optimal tuple that minimizes cost while satisfying the risk constraint.
The PAC guarantee carries over: with probability at least $1 - \alpha$, the selected thresholds satisfy $R(u_1, \ldots, u_{K-1}) \le \epsilon$.
The main challenge is computational: the number of threshold tuples grows exponentially with $K$, which increases calibration cost and may require larger calibration sets.
In practice, $K = 3$ or $K = 4$ sources are typically sufficient.

\section{Score}
\label{app:score}

A key component of our hybrid annotation system is the score $U(x)$ that measures the uncertainty of the annotation system for input $x$.
This score is used to decide which annotation source to use: examples with low scores can be handled by cheaper models, while examples with high scores should be sent to more accurate sources.

Intuitively, this uncertainty may relate to the difficulty of the input, but our theory does not require such an assumption.
A more accurate score leads to a more efficient calibration process.
In practice, model uncertainty serves as a natural proxy for this purpose \citep{huang2025look,xiong2023can,hu2024uncertainty}.
When a model is uncertain about its prediction, this often indicates that the input is inherently ambiguous, lies near decision boundaries, or requires deeper reasoning that the model cannot perform reliably.
Such examples are precisely those that benefit from more capable annotation sources.

We consider two main approaches to computing the score for large language models, both of which can be used within our framework.

\textbf{Logits-based score}.
For autoregressive language models that generate sequences token by token, we can use the model's output probabilities to compute the score \citep{guo2017calibration}.
Let $y = (y_1, y_2, \ldots, y_l)$ be a generated sequence of length $l$.
The model produces a probability distribution over the next token at each step: $\mathbb{P}(y_j \mid y_1, \ldots, y_{j-1}, x)$.
We define the logits-based score as
\begin{equation*}
\label{eq:logits-score}
U_{\text{logits}}(x) = 1 - \frac{1}{l} \sum_{j=1}^l \mathbb{P}(y_j \mid y_1, \ldots, y_{j-1}, x)
\end{equation*}
where $y_j$ is the token actually generated by the model at position $j$.
This score measures the average confidence of the model across all generated tokens.
When the model is very confident about each token (high probabilities), the score is low.
When the model is less confident (low probabilities), the score is high, indicating an example that may benefit from a more capable annotation source.

This approach has several advantages: it is model-agnostic, requires no additional computation beyond standard inference, and provides a fine-grained measure at the token level.
However, it requires access to the model's logits, which may not always be available for API-based models.

\textbf{Verbalized score}.
An alternative approach is to ask the model to explicitly express its confidence in natural language \citep{yang2024verbalized,tian2023just}.
We can prompt the model to output a confidence level along with its prediction.
We then map these verbalized confidence levels to numerical scores in $[0, 1]$.

For example, we might use a prompt like: `Please provide your answer and indicate your confidence level (high/medium/low).'
We then parse the model's response to extract the confidence level and map it to a score: high confidence $\to$ low score, low confidence $\to$ high score.

This approach has the advantage of working with black-box API models where we do not have access to logits.
It also allows the model to use its reasoning capabilities to assess the input, potentially capturing semantic information that may not be reflected in token-level probabilities.
However, it requires careful prompt engineering and may add overhead to the inference process.

\textbf{Alternative scores}.
Our framework is flexible and can work with other methods for computing the score.
One alternative is ensemble-based scoring, where we run the model multiple times with different random seeds or sampling parameters and measure the disagreement among the outputs, with higher disagreement indicating higher uncertainty.
Another alternative is calibrated scoring, where we use a separate calibration model trained to predict the probability that the main model's output is correct based on features of the input and output.
We can also use hybrid approaches that combine multiple scoring methods to get a more robust estimate.

The choice of scoring method depends on the specific application, the available model APIs, and the computational budget.
In our experiments, we focus on logits-based and verbalized scores as they are widely applicable and easy to implement.

\section{Compared baselines}
\label{sec:baselines}
In this work, we compare HyPAC against PAC labeling \citep{candes2025probably}, Cascaded Selective Evaluation (CSE) \citep{jung2024trust}, and a heuristic baseline.
PAC Labeling leverages a confidence threshold to route each example either to an LLM for automatic labeling or to an expert annotator.
CSE performs sequential selective prediction by escalating inputs to increasingly expensive models based on confidence thresholds, without providing an explicit finite-sample guarantee on the final annotation error.
For the heuristic baseline, we select hyperparameters $\hat{u}_1$ and $\hat{u}_2$ to minimize annotation cost while ensuring that the empirical risk on the calibration data does not exceed $\epsilon$.

\paragraph{Probably Approximately Correct (PAC) Labels \citep{candes2025probably}.}
PAC labeling is a selective annotation framework that combines inexpensive model-generated labels with expert annotations while providing a dataset-level accuracy guarantee. Given a target error level $\epsilon$ and tolerance $\alpha$, the method selectively escalates high-uncertainty samples to human annotators, ensuring that the overall labeling error is bounded with high probability.
The original PAC labeling formulation assumes a fixed dataset and does not rely on a separate calibration set. To make it compatible with our framework, we introduce a slight modification and reformulate it as a two-stage PAC labeling procedure.

\begin{algorithm}[t]
\caption{Two-Stage PAC Labeling with Uncertainty Thresholding}
\label{alg:pac_labeling}
\textbf{Input:} Calibration set $\mathcal I_{\text{cal}}$, test set $\mathcal I_{\text{test}}$,
LLM $\tilde f$, expert $f$, loss $\ell$,
error tolerance $\epsilon$, confidence level $\alpha$. \\
\textbf{Output:} Composite predictor $\hat f$.
\begin{algorithmic}[1]
\STATE For all $i \in \mathcal I_{\text{cal}}$, compute $\tilde y_i = \tilde f(x_i)$ and uncertainty scores $U_i$.
\STATE For each $u \in \{U_i : i \in \mathcal I_{\text{cal}}\}$, compute $\hat L_u(\alpha)$ as in Eq.(\ref{eq:ucb}).
\STATE Set
\[
\hat u = \max\{u : \hat L_u(\alpha) \le \epsilon\}.
\]
\STATE For all $i \in \mathcal I_{\text{test}}$, compute $\tilde y_i = \tilde f(x_i)$ and $U_i$.
\STATE Define
\[
\hat f(x_i) =
\begin{cases}
\tilde f(x_i), & U_i < \hat u, \\
f(x_i), & U_i \ge \hat u.
\end{cases}
\quad \forall i \in \mathcal I_{\text{test}}.
\]
\STATE \textbf{Return} $\hat f$.
\end{algorithmic}
\end{algorithm}

\paragraph{Cascaded Selective Evaluation (CSE).}
The CSE guarantees a target level of loss control for LLM response while reducing evaluation cost by escalating only uncertain cases to larger models.
\footnote{
Throughout this work, we use $\epsilon$ to denote the target error level, and $\alpha$ to denote the tolerance level controlling
the probability with which this guarantee holds. This notation is different from the CSE paper.
}
Under the i.i.d.\ data assumption, Cascaded Selective Evaluation (CSE) admits a rigorous theoretical guarantee: with probability $1-\alpha$,
\[
\mathbb{P}\!\left(f_{\mathrm{CSE}}(x)\neq y_{\mathrm{human}}
\;\middle|\; x \text{ is labeled by } \tilde{f}_1 \text{ or } \tilde{f}_2\right)\le \delta.
\]
Note that in the original CSE implementation, expert annotations are not considered. All data instances that are not annotated by the LLM are directly escalated. 
To make the method compatible with our setup, we use expert annotations for escalated instances.

\paragraph{CoAnnotating \citep{li2023coannotating}.}
We consider a cost-aware heuristic that selects thresholds by directly optimizing empirical performance on a calibration set, without formal risk control. 
The method partitions examples according to their uncertainty scores and searches for threshold pairs that minimize total annotation cost while keeping the empirical loss below a target level. 
The selected thresholds are then applied to the test set using the same escalation rule. 
The full procedure is summarized in Algorithm~\ref{alg:heuristic_mean_baseline}.

\begin{algorithm}[t!]
\caption{Heuristic Mean Baseline}
\label{alg:heuristic_mean_baseline}
\begin{algorithmic}[1]
\REQUIRE
Calibration set $\mathcal{D}_{\mathrm{cal}}: \{(x_i, y_i)\}_{i\in[m]}$, test set $\mathcal{D}_{\mathrm{test}}$: $\{(x_{m+j}, y_{m+j})\}_{j\in[n]}$; \\
Three annotators: non-thinking model $\tilde{f}_1$, thinking model $\tilde{f}_2$, and expert annotator; \\
Uncertainty scores $U(x)\in[0,1]$; \\
Per-annotator cost array $\mathrm{cost}_i(x)$; \\
Target error threshold $\epsilon$
\ENSURE
Predicted labels on $\mathcal{D}_{\mathrm{test}}$

\STATE \textbf{Calibration phase:}
\STATE Construct a grid $\mathcal{U}$ from the unique values of $\{U(x):x\in\mathcal{D}_{\mathrm{cal}}\}$
\STATE Initialize an empty candidate set $\mathcal{T}$

\FOR{each $u_1 \in \mathcal{U}$}
    \STATE Search for the largest $u_2 \ge u_1$ such that
    \[
    \text{Mean}(\{L(x_i, u_1, u_2)\}_{i\in[m]}) \le \epsilon
    \]
    \IF{such $u_2$ exists}
        \STATE Add $(u_1,u_2)$ to $\mathcal{T}$
    \ELSE
        \STATE \textbf{break}
    \ENDIF
\ENDFOR

\STATE Select $(u_1^\star,u_2^\star)\in\mathcal{T}$ minimizing the total calibration cost

\STATE \textbf{Evaluation phase:}
\FOR{each $(x, y)\in\mathcal{D}_{\mathrm{test}}$}
    \IF{$U(x)\le u_1^\star$}
        \STATE Output prediction of $\tilde{f_1}$
    \ELSIF{$u_1^\star < U(x)\le u_2^\star$}
        \STATE Output prediction of $\tilde{f}_2$
    \ELSE
        \STATE Output expert label
    \ENDIF
\ENDFOR
\end{algorithmic}
\end{algorithm}

\section{Implementation details}
\label{sec:implementation}
\paragraph{Hyperparameter settings of LLMs}
\label{para:hyperparameter}
In this study, we configure the decoding parameters as follows: for Qwen/Qwen3-4B-Instruct-2507, we set Temperature=0.7, TopP=0.8, TopK=20, and MinP=0; for Qwen/Qwen3-4B-Thinking-2507, we set Temperature=0.6, TopP=0.95, TopK=20, and MinP=0.
For Llama/Llama-3.1-8B-Instruct, we set Temperature=0.7, TopP=0.8, and MinP=0; for DeepSeek/DeepSeek-R1-Distill-Llama-8B, we set Temperature=0.6, TopP=0.95, and MinP=0.
For Qwen/Qwen2.5-32B-Instruct, we set Temperature=0.7, TopP=0.8, and MinP=0; for DeepSeek/DeepSeek-R1-Distill-Qwen-32B, we set Temperature=0.6, TopP=0.95, and MinP=0.
For experiments on the MATH-500 dataset, we additionally apply a presence penalty of 1.5 to avoid repetitive generation.
Experiments were run on NVIDIA GeForce RTX 4090, NVIDIA L40, and NVIDIA RTX Pro 6000.

\paragraph{Dataset details}
Table~\ref{tab:datasets_split} summarizes the datasets employed in our experiments, together with their corresponding splitting strategies.
For each dataset, we report its type, overall size, and the partitioning into calibration and test sets.
In Table~\ref{tab:accuracy}, we report the accuracy of LLMs on these annotation datasets.

\begin{table*}[t]
\centering
\caption{Accuracy (\%) of different models on annotation datasets.}
\label{tab:accuracy}

\setlength{\tabcolsep}{2pt} 
\resizebox{1.13\textwidth}{!}{%
\begin{tabular}{lcccccc}
\toprule
Dataset / LLM
& Qwen3-4B-Instruct-2507
& Qwen3-4B-Thinking-2507
& Llama-3.1-8B-Instruct
& DeepSeek-R1-Distill-Llama-8B
& Qwen2.5-32B-Instruct
& DeepSeek-R1-Distill-Qwen-32B \\
\midrule
MMLU-Redux
& 79.83 & 81.57 & 64.77 & 67.83 & 79.73 & 84.93 \\
MATH-500
& 91.80 & 96.00 & 46.80 & 87.20 & 80.20 & 92.60 \\
MATH-L5
& 91.12 & 94.73 & 20.53 & 80.72 & 56.17 & 93.62 \\
ZebraLogic
& 79.10 & 89.00 & 13.2 & 34.80 & 24.50 & 68.60 \\
HumanEval$^+$
& 92.68 & 97.56 & 92.68 & 93.90 & 82.32 & 99.39 \\
\bottomrule
\end{tabular}
}

\end{table*}

\begin{table}[t]
\centering
\caption{Details of datasets and splitting settings for annotation experiments.}
\label{tab:datasets_split}
\begin{tabular}{llccc}
\toprule
Dataset & Dataset Type & Dataset Size & Split Setting & Size \\
\midrule
\multirow{2}{*}{MMLU-Redux} 
& \multirow{2}{*}{General Task} 
& \multirow{2}{*}{3000} 
& Calibration & 500 \\
& & & Test & 2500 \\
\midrule
\multirow{2}{*}{MATH-500} 
& \multirow{2}{*}{Math Reasoning} 
& \multirow{2}{*}{500} 
& Calibration & 300 \\
& & & Test & 200 \\
\midrule
\multirow{2}{*}{MATH-L5} 
& \multirow{2}{*}{Math Reasoning} 
& \multirow{2}{*}{721} 
& Calibration & 300 \\
& & & Test & 421 \\
\midrule
\multirow{2}{*}{Zebra-Logic} 
& \multirow{2}{*}{Text Reasoning} 
& \multirow{2}{*}{750} 
& Calibration & 450 \\
& & & Test & 300 \\
\midrule
\multirow{2}{*}{HumanEval+} 
& \multirow{2}{*}{Coding Task} 
& \multirow{2}{*}{164} 
& Calibration & 100 \\
& & & Test & 64 \\
\bottomrule
\end{tabular}
\end{table}

\paragraph{Prompts for verbalized confidence}
In Table~\ref{tab:verbalized_score_prompt}, we present the prompt used to elicit the verbalized confidence scores. After ten trials, we obtained the average confidence score and defined the verbalized uncertainty score as $1$ minus this average confidence.
For PAC labeling baseline on $\text{HumanEval}^+$, we find that using this prompt, we fail to extract the verbalized confidence for Qwen3-4B-Thinking-2507.
Therefore, on $\text{HumanEval}^+$, we instead use Qwen3-4B-Instruct-2507 to obtain verbalized confidence for the outputs generated by Qwen3-4B-Thinking-2507.

\begin{table}
\caption{Prompt for the verbalized confidence scores.}
\label{tab:verbalized_score_prompt}
\renewcommand{\arraystretch}{1.4}
\rowcolors{1}{Gray}{Gray}

\resizebox{\textwidth}{!}{
\begin{tabular}{!{\vrule width 1.2pt} p{0.95\textwidth} !{\vrule width 1.2pt}}
\textbf{System prompt:} You are a reasoning assistant. For each question and proposed answer, you must estimate how likely the proposed answer is correct.\\

\textbf{User prompt:}\\
Question: \{QUESTION\}\\
Answer: \{ANSWER\}\\
Provide a probability (between 0.0 and 1.0) that your answer is correct. Only output the probability.\\
\end{tabular}
}

\end{table}

\paragraph{API pricing parameters}
We evaluate API cost functions using two model pairs: (i) Qwen2.5-32B-Instruct and DeepSeek-R1-Distill-Qwen-32B, and (ii) Llama-3.1-70B-Instruct and DeepSeek-R1-Distill-Llama-70B.
For the Qwen2.5-32B-Instruct and DeepSeek-R1-Distill-Qwen-32B models, both input and output token prices are identical under Aliyun’s pricing scheme, with input tokens priced at 0.002 Chinese Yuan per thousand tokens and output tokens priced at 0.006 Chinese Yuan per thousand tokens.
For the 70B model pair, we adopt the API pricing provided by Deeplnfra.
The Llama-3.1-70B-Instruct model is priced at \$0.40 per thousand tokens for both input and output, while DeepSeek-R1-Distill-Llama-70B is priced at \$0.80 per thousand tokens for both input and output.


\paragraph{Hyperparameter setting for empirical comparison between HyPAC and CSE.}
We use dataset-specific hyperparameters for CSE to attain comparable errors to HyPAC.
For logits-based uncertainty score, on MMLU-Redux, we set $\alpha=0.07$ and $\epsilon=0.15$; on MATH-500, we set $\alpha=0.08$ and $\epsilon=0.10$; on MATH-L5, we set $\alpha=0.1$ and $\epsilon=0.14$; on Zebra-Logic, we set $\alpha=0.10$ and $\epsilon=0.10$; on $\text{HumanEval}^+$, we set $\alpha=0.10$ and $\epsilon=0.12$.
For verbalized uncertainty score, on MMLU-Redux, we set $\alpha=0.1$ and $\epsilon=0.15$; on MATH-500, we set $\alpha=0.1$ and $\epsilon=0.14$; on MATH-L5, we set $\alpha=0.1$ and $\epsilon=0.1$; on Zebra-Logic, we set $\alpha=0.10$ and $\epsilon=0.10$; on $\text{HumanEval}^+$, we set $\alpha=0.05$ and $\epsilon=0.05$.

\section{Extensive Study}

\subsection{Performance of HyPAC under the semantic cosine-distance loss function}
\label{subsec:semantic_loss}

\paragraph{HyPAC tightly controls the annotation error under the semantic cosine-distance loss function.}
In our main experiment, we use the binary 0--1 loss function.
In this subsection, we adopt the semantic cosine-distance loss function, which measures the semantic similarity
between outputs in the embedding space.
Formally, given the human-annotated labels y and the LLM-annotated label $\hat{y}$, we obtain their embeddings $v_y$ and $v_{\hat{y}}$ using a semantic embedding model, and define the loss as: 
\begin{equation}
\label{eq:semantic_loss}
    \ell(y_i, \hat{y}_i) = 1 - \frac{v_y \cdot v_{\hat{y}}}{\|v_y\| \|v_{\hat{y}}\|}.
\end{equation}
For the semantic embedding model, we adopt ``Qwen/Qwen3-Embedding-8B''~\citep{yang2025qwen3}.
We conduct experiments on MATH-500, which provides ground-truth chains of thought for each question.
We compute the cosine distance between the LLM-generated output and the ground-truth chain of thought to measure their semantic similarity.

In Figure~\ref{fig:semantic_math500_results}, we present the annotation results of HyPAC on MATH-500 under the semantic cosine-distance loss function. The results demonstrate that HyPAC tightly controls the annotation loss while substantially reducing the annotation cost under the semantic loss. Notably, the annotation error is also low, indicating that the semantic loss provides a reliable measure of annotation quality. 
Overall, these empirical results show that HyPAC achieves effective cost reduction without sacrificing annotation accuracy under the semantic loss function.

\begin{figure}[t]
    \centering

    \begin{subfigure}[b]{0.32\linewidth}
        \centering
        \includegraphics[height=3.2cm,width=\linewidth]{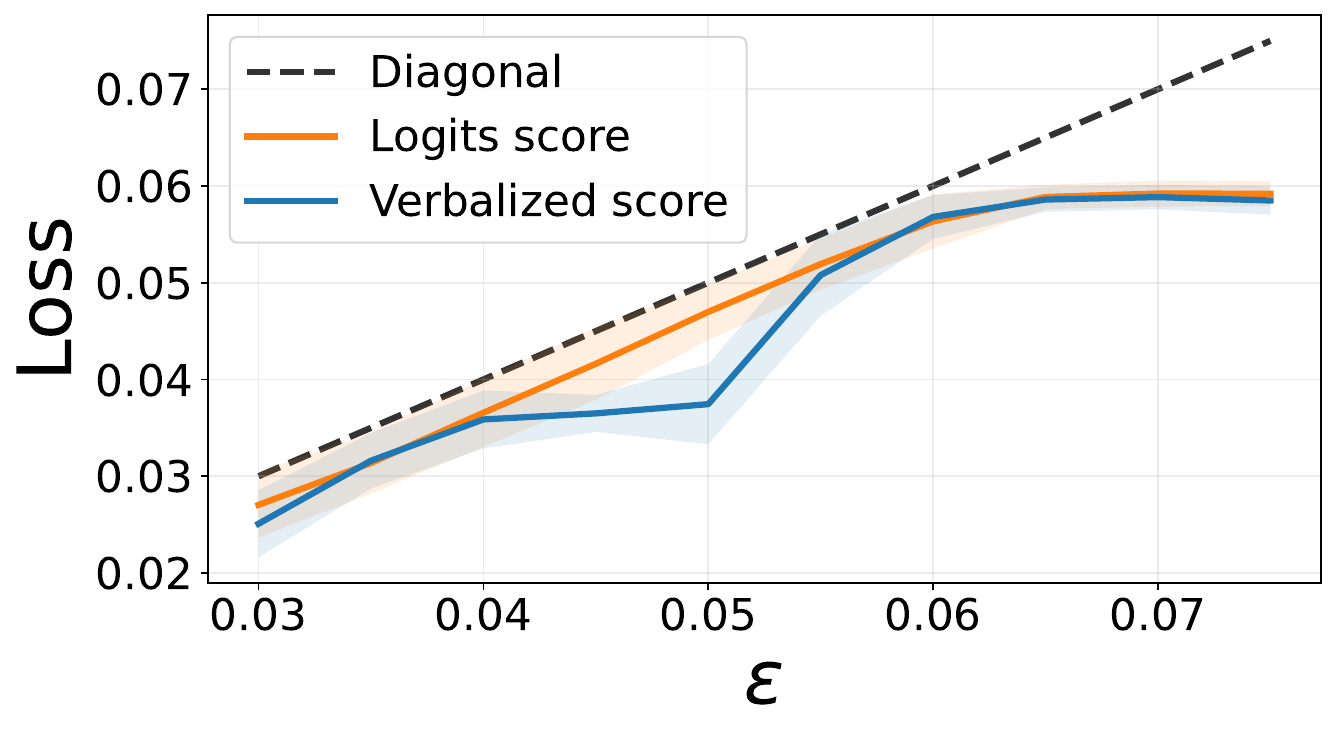}
        \caption{Semantic loss}
        \label{fig:math500_loss}
    \end{subfigure}
    \hfill
    \begin{subfigure}[b]{0.32\linewidth}
        \centering
        \includegraphics[height=3.2cm,width=\linewidth]{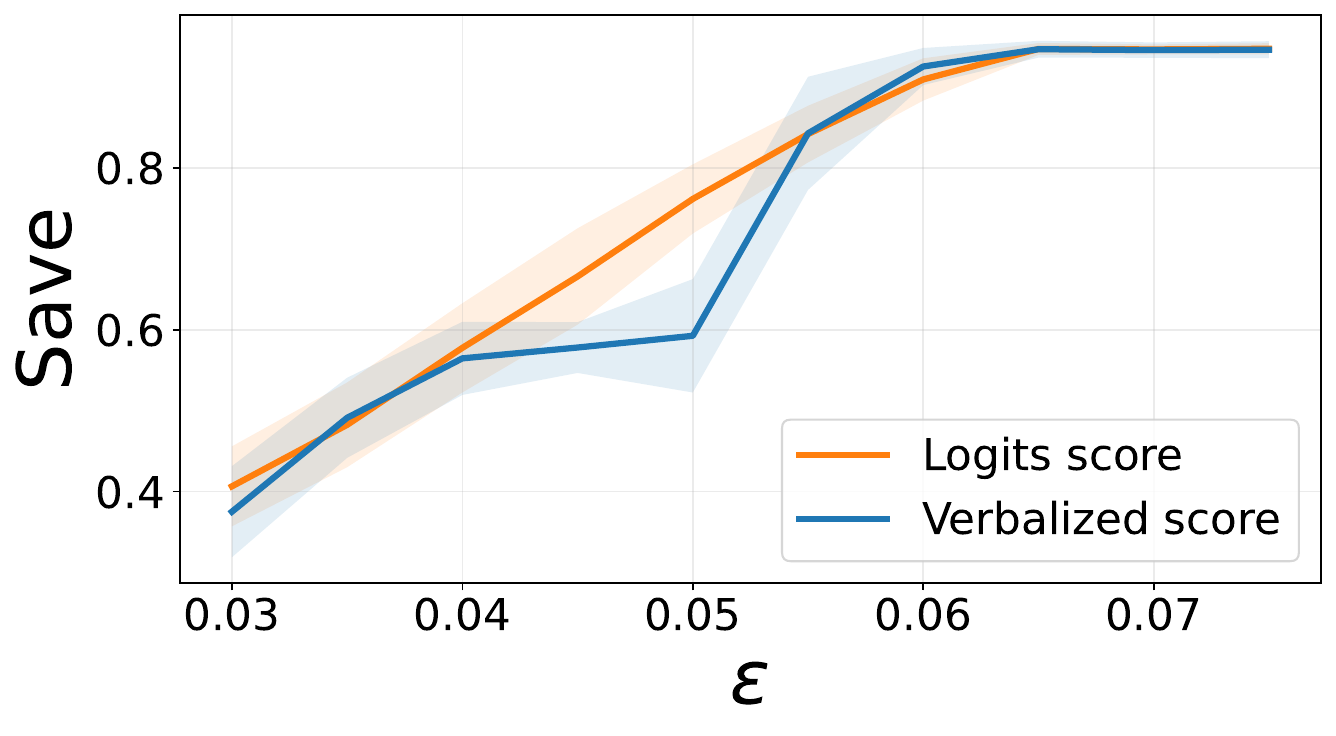}
        \caption{Cost saving}
        \label{fig:math500_cost}
    \end{subfigure}
    \hfill
    \begin{subfigure}[b]{0.32\linewidth}
        \centering
        \includegraphics[height=3.2cm,width=\linewidth]{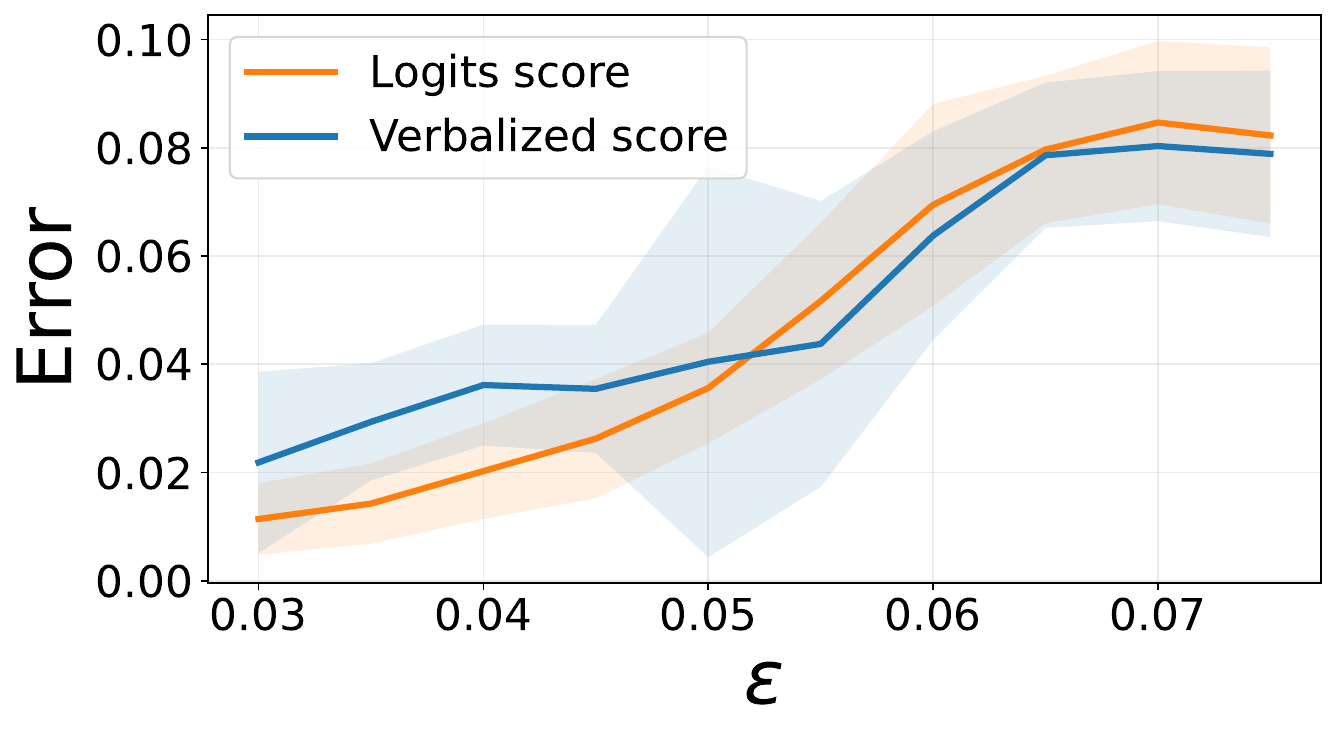}
        \caption{Error}
        \label{fig:math500_error}
    \end{subfigure}

    \caption{
    \textbf{HyPAC achieves higher cost savings while maintaining controlled semantic loss and error on MATH-500.}
    Loss, cost savings, and error of HyPAC with semantic loss on MATH-500 at confidence level $\alpha=0.05$.
    We use Qwen3-4B-Instruct-2507 as the non-thinking model,
    Qwen3-4B-Thinking-2507 as the thinking model, and Qwen3-8B-Embedding as the embedding model.
    }
    \label{fig:semantic_math500_results}
    \vspace{-6pt}
\end{figure}

\subsection{Performance of HyPAC under the API cost function}
\label{subsec:api_cost_performance}

\paragraph{HyPAC reduces the annotation cost under the API cost function.}
We evaluate whether HyPAC effectively reduces annotation cost under API-based cost functions on MMLU-Redux. Figure~\ref{fig:api_cost_function} presents the annotation results of HyPAC and PAC Labeling under the API cost function.
The results show that HyPAC consistently achieves substantially higher cost savings while maintaining controlled annotation error compared to PAC Labeling, demonstrating the superiority of HyPAC.
Overall, these results demonstrate that HyPAC effectively reduces annotation cost under the API cost function while preserving tight error control.

\begin{figure*}
    \centering

    \begin{subfigure}[b]{0.48\textwidth}
        \centering
        \includegraphics[height=4.0cm,width=7cm]{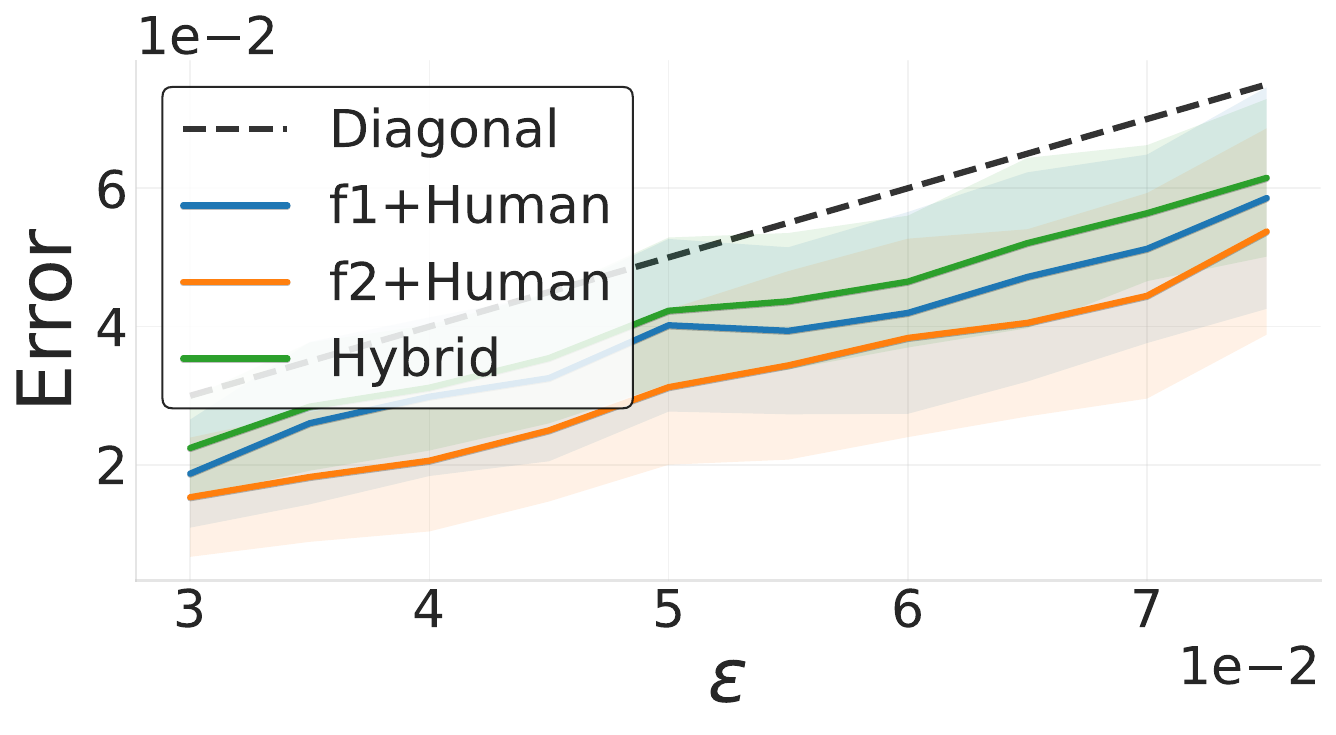}
    \end{subfigure}
    \hfill
    \begin{subfigure}[b]{0.48\textwidth}
        \centering
        \includegraphics[height=4.0cm,width=7cm]{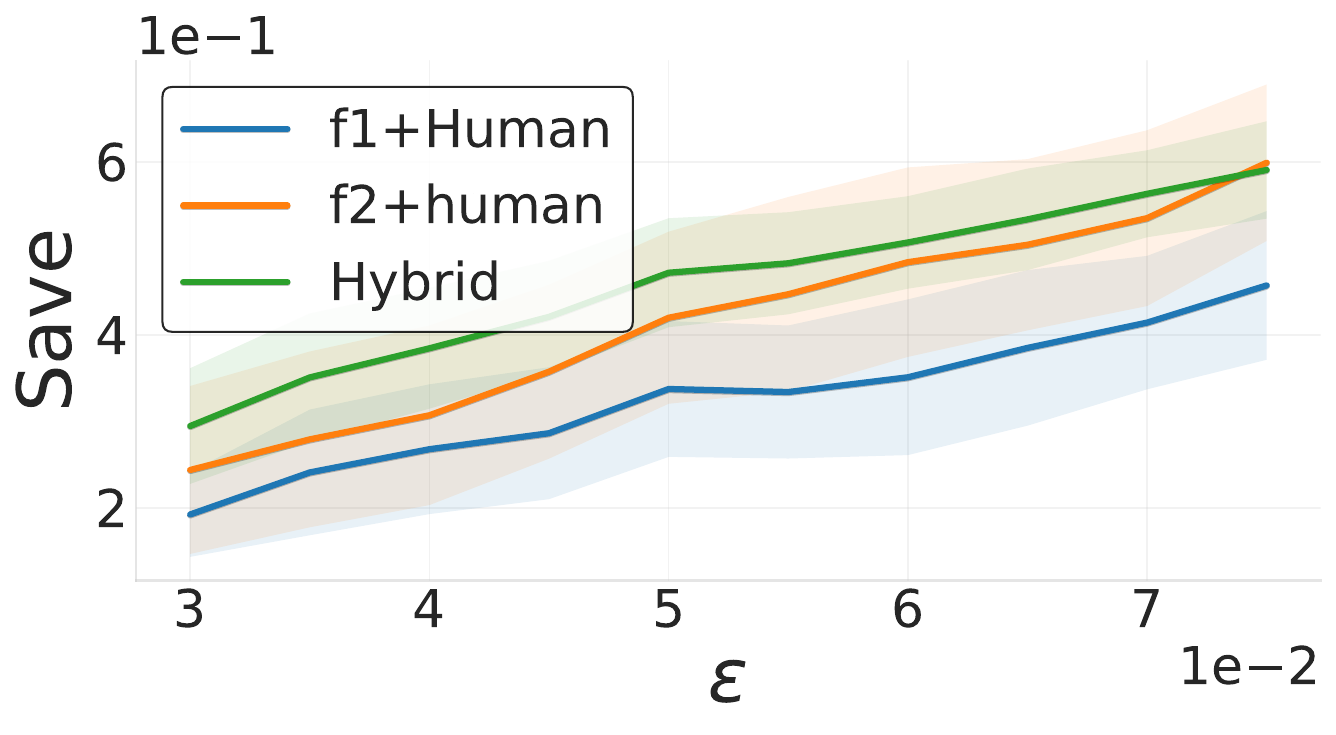}
    \end{subfigure}

    {(a) Qwen-32B model pair}

    \vspace{6pt}

    \begin{subfigure}[b]{0.48\textwidth}
        \centering
        \includegraphics[height=4.0cm,width=7cm]{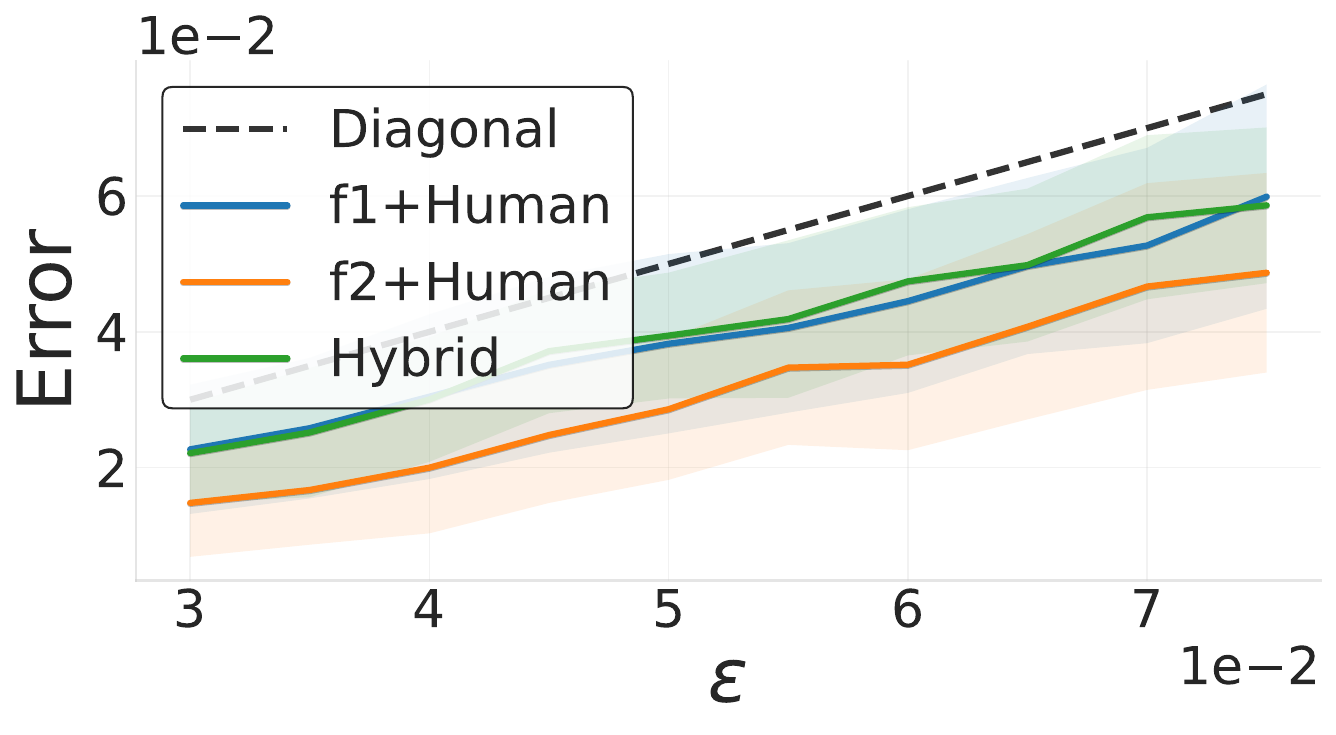}
        \label{fig:mmlu_loss}
    \end{subfigure}
    \hfill
    \begin{subfigure}[b]{0.48\textwidth}
        \centering
        \includegraphics[height=4.0cm,width=7cm]{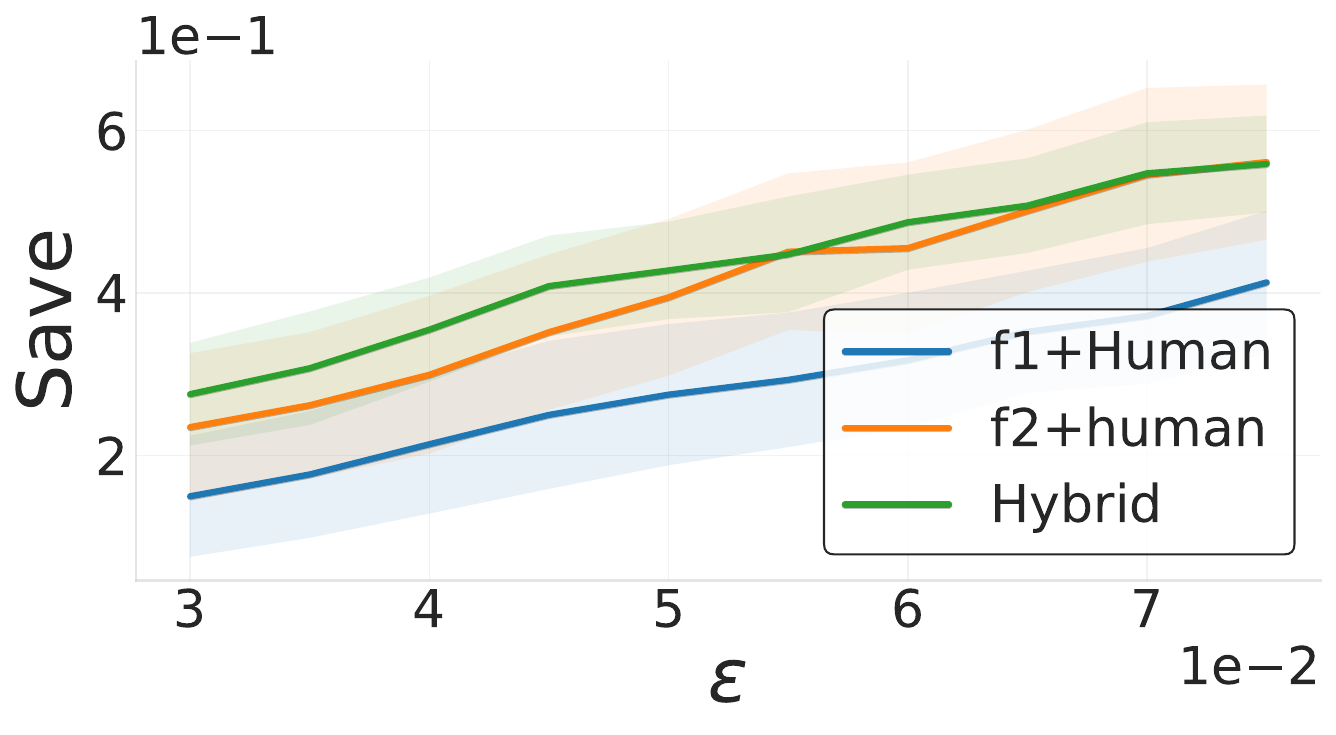}
        \label{fig:mmlu_cost}
    \end{subfigure}

    {(b) Llama-70B model pair}

    \caption{
    \textbf{HyPAC achieves higher cost savings than PAC Labeling while maintaining controlled error under the API cost function.}
    Error and cost savings of HyPAC compared against PAC Labeling baselines under the API cost function.
Experiments are conducted on MMLU-Redux with logits-based uncertainty score under a confidence level of $\alpha = 0.05$.
The shaded areas represent standard deviations.
    }
    \label{fig:api_cost_function}
    \vspace{-8pt}
\end{figure*}

\subsection{Ablation Study on Calibration Set Size}
\label{subsec:calibration_set_size}
\paragraph{Robust to calibration set size}
We conduct experiments to investigate the stability of error and cost savings under different calibration set sizes.
Specifically, we repeat the experiments with varying calibration ratios to examine how the size of the calibration set influences performance.
The results are presented in Figure~\ref{fig:cal_ratio_math500_mmlu}. 
The results show that HyPAC tightly controls the annotation error below the target level, while a larger calibration size leads to slightly smaller shaded areas, indicating more stable annotation results.
In addition, we observe that using a larger calibration set size can slightly increase the cost savings.

\begin{figure*}
    \centering

    \begin{subfigure}[b]{0.48\textwidth}
        \centering
        \includegraphics[height=3.0cm,width=6.2cm]{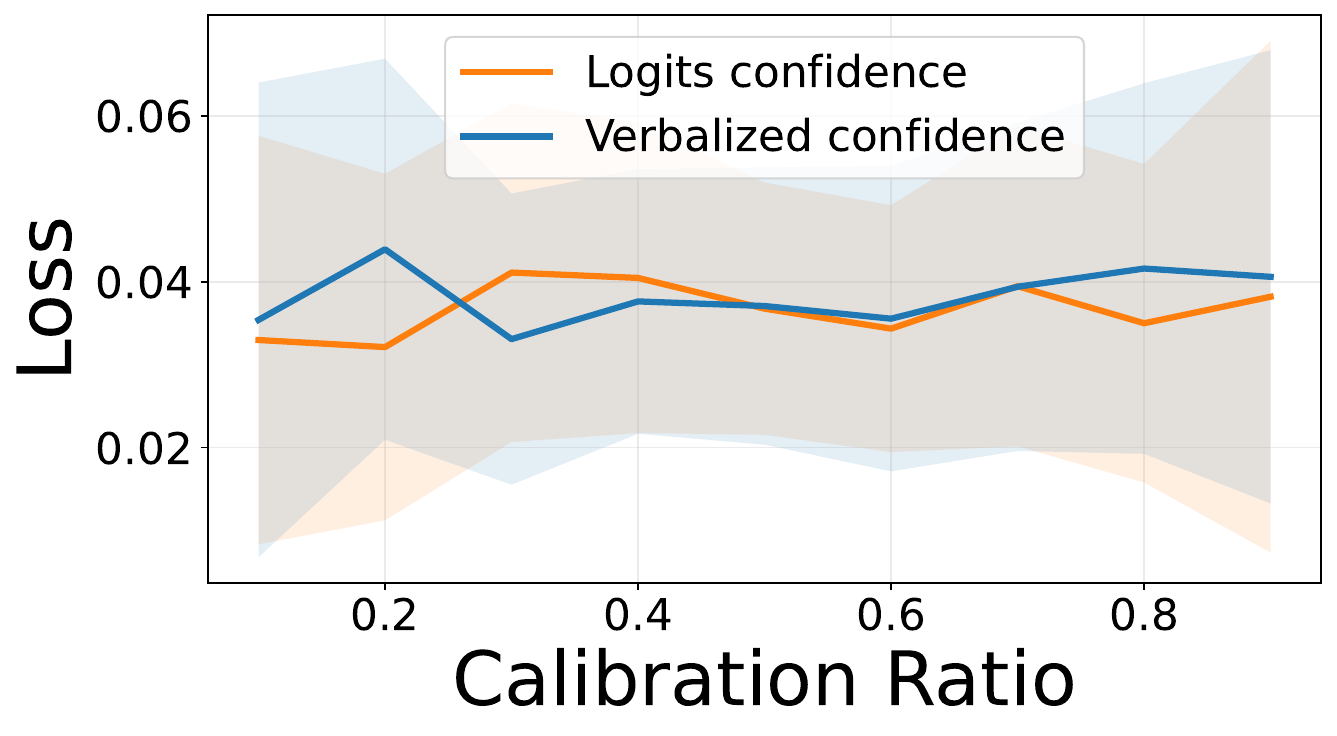}
    \end{subfigure}
    \hfill
    \begin{subfigure}[b]{0.48\textwidth}
        \centering
        \includegraphics[height=3.0cm,width=6.2cm]{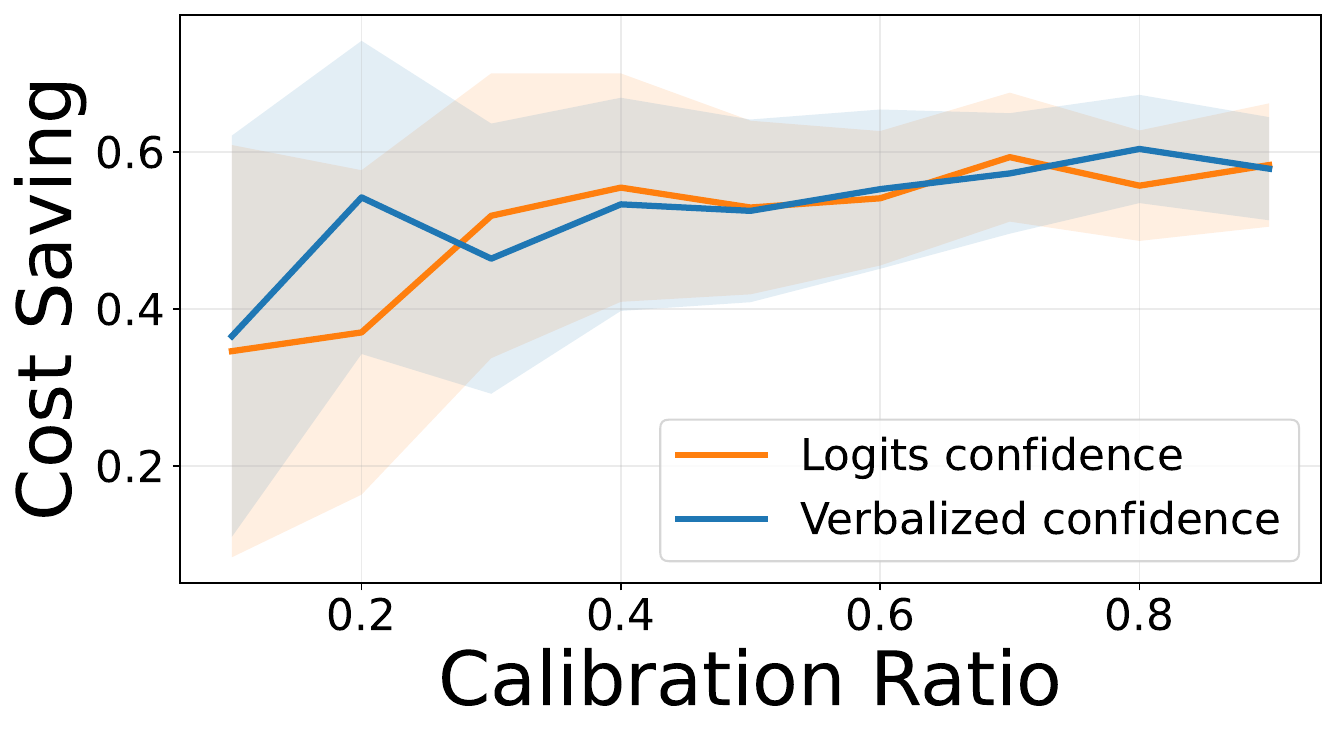}
    \end{subfigure}

    {(a) MATH-500}

    \vspace{6pt}

    \begin{subfigure}[b]{0.48\textwidth}
        \centering
        \includegraphics[height=3.0cm,width=6.2cm]{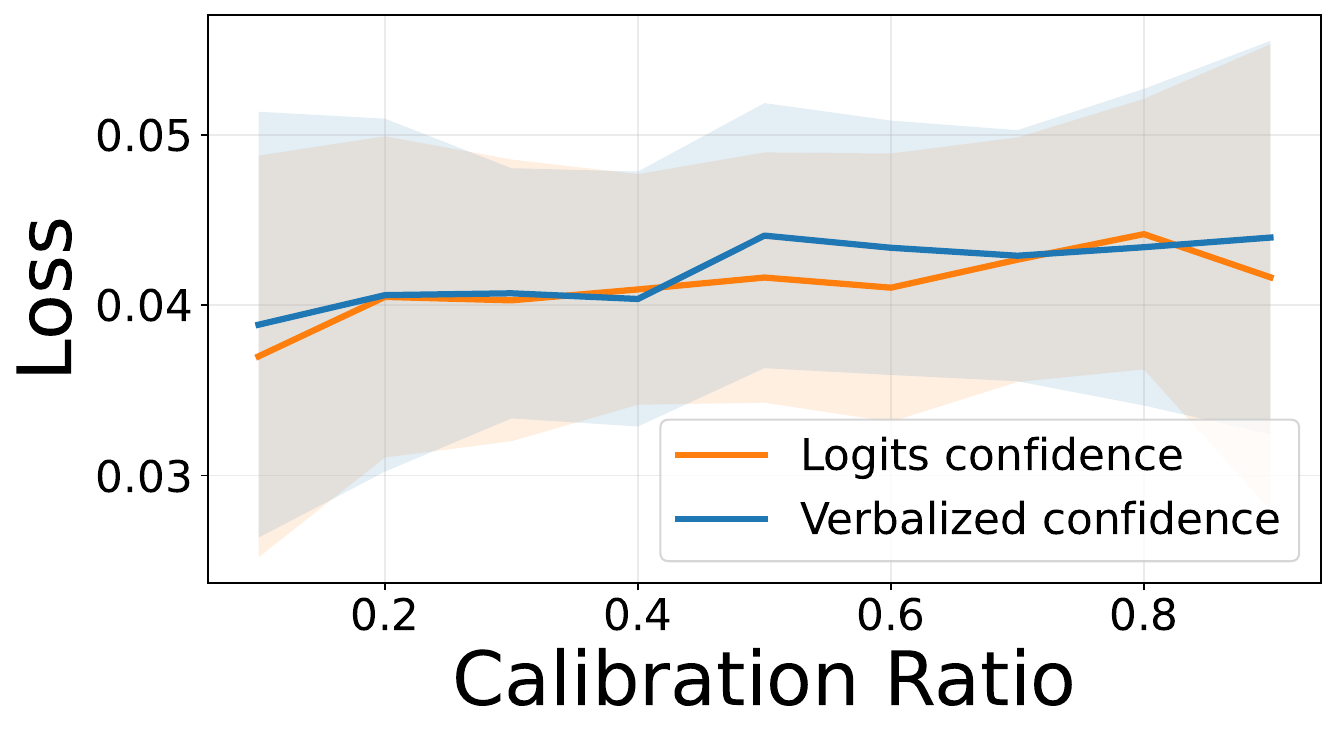}
    \end{subfigure}
    \hfill
    \begin{subfigure}[b]{0.48\textwidth}
        \centering
        \includegraphics[height=3.0cm,width=6.2cm]{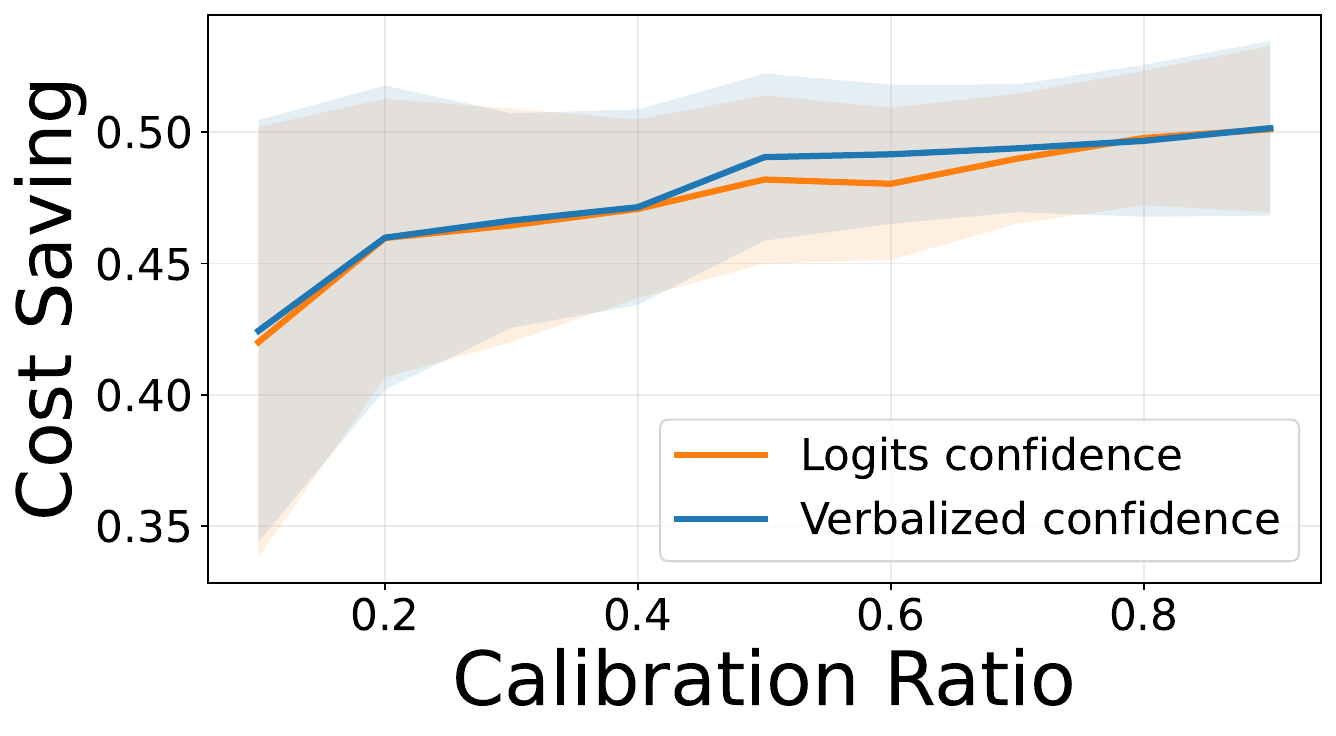}
    \end{subfigure}

    {(b) MMLU-Redux}

    \caption{
    \textbf{HyPAC maintains stable performance across different calibration set sizes.}
    Error and cost savings of HyPAC under different calibration set sizes with confidence level $\alpha=0.05$.
    The shaded areas indicate one standard deviation.
    }
    \label{fig:cal_ratio_math500_mmlu}
    \vspace{-8pt}
\end{figure*}

\subsection{Performance of HyPAC under finite-sample UCBs}
\label{subsec:different_ucb}
\paragraph{Betting-based UCB achieves the best finite-sample trade-off.}
Figure~\ref{fig:three-parallel} compares the annotation performance of HyPAC under four different upper confidence bounds (UCBs): the CLT-based UCB, Hoeffding-based UCB, empirical Bernstein UCB, and betting-based UCB.
A detailed description of these UCB constructions is provided in Appendix~\ref{app:ucb-validity}.
Among all the finite-sample UCBs, the betting-based UCB achieves the best cost--quality trade-off, as it yields the tightest control of the $(1-\alpha)$-quantile loss and the highest cost savings.
In addition, we observe that the $(1-\alpha)$-quantile loss under the CLT-based UCB occasionally exceeds the target threshold~$\epsilon$.
This is because the CLT-based UCB is only asymptotically valid and does not have a guarantee in finite samples.
Overall, these results show the advantage of betting-based UCB in finite-sample regimes.

\begin{figure}[t]
    \centering
    \begin{minipage}{0.32\linewidth}
        \centering
        \includegraphics[width=\linewidth]{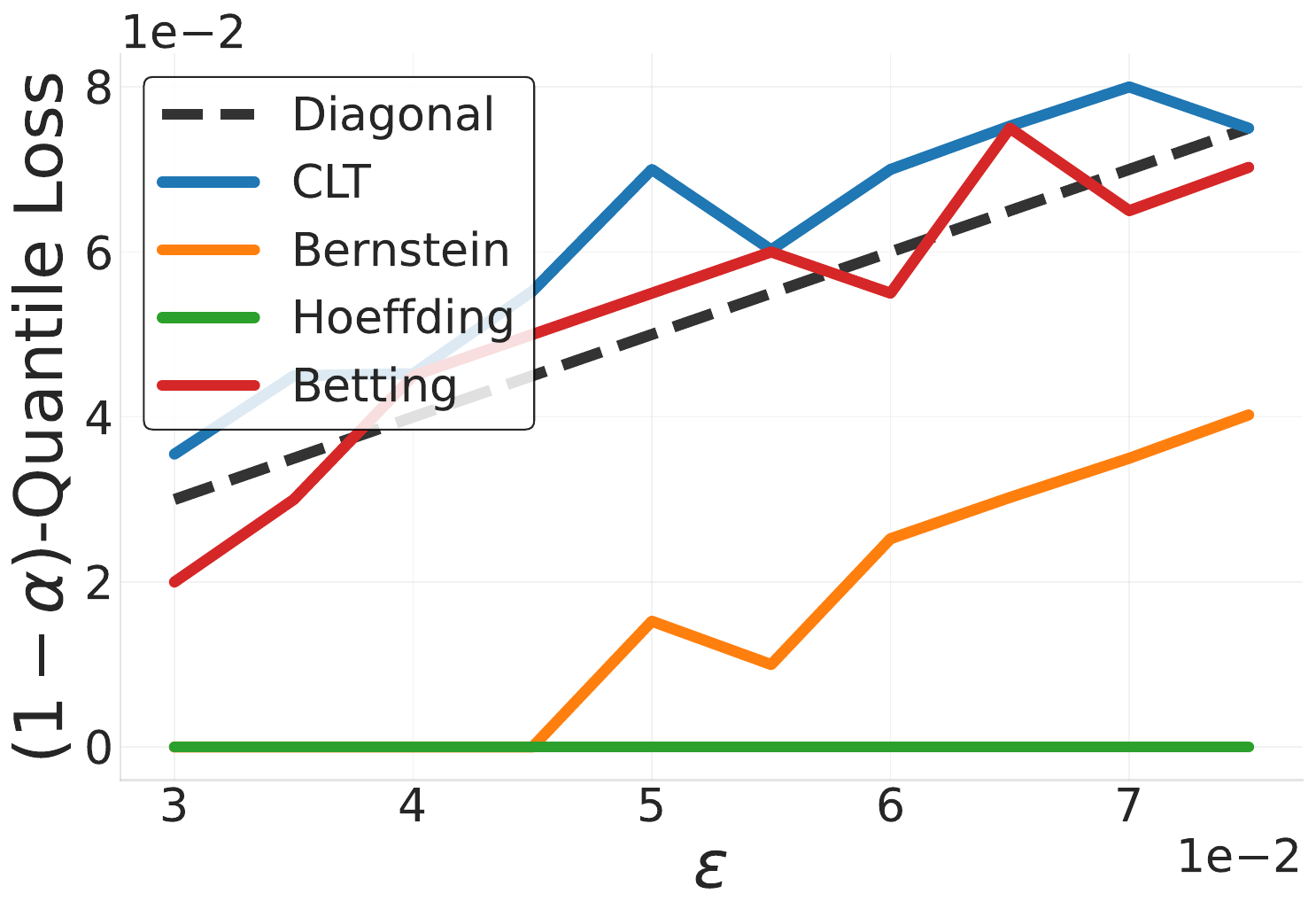}
        \caption*{$95\%$ Quantile loss}
    \end{minipage}\hfill
    \begin{minipage}{0.32\linewidth}
        \centering
        \includegraphics[width=\linewidth]{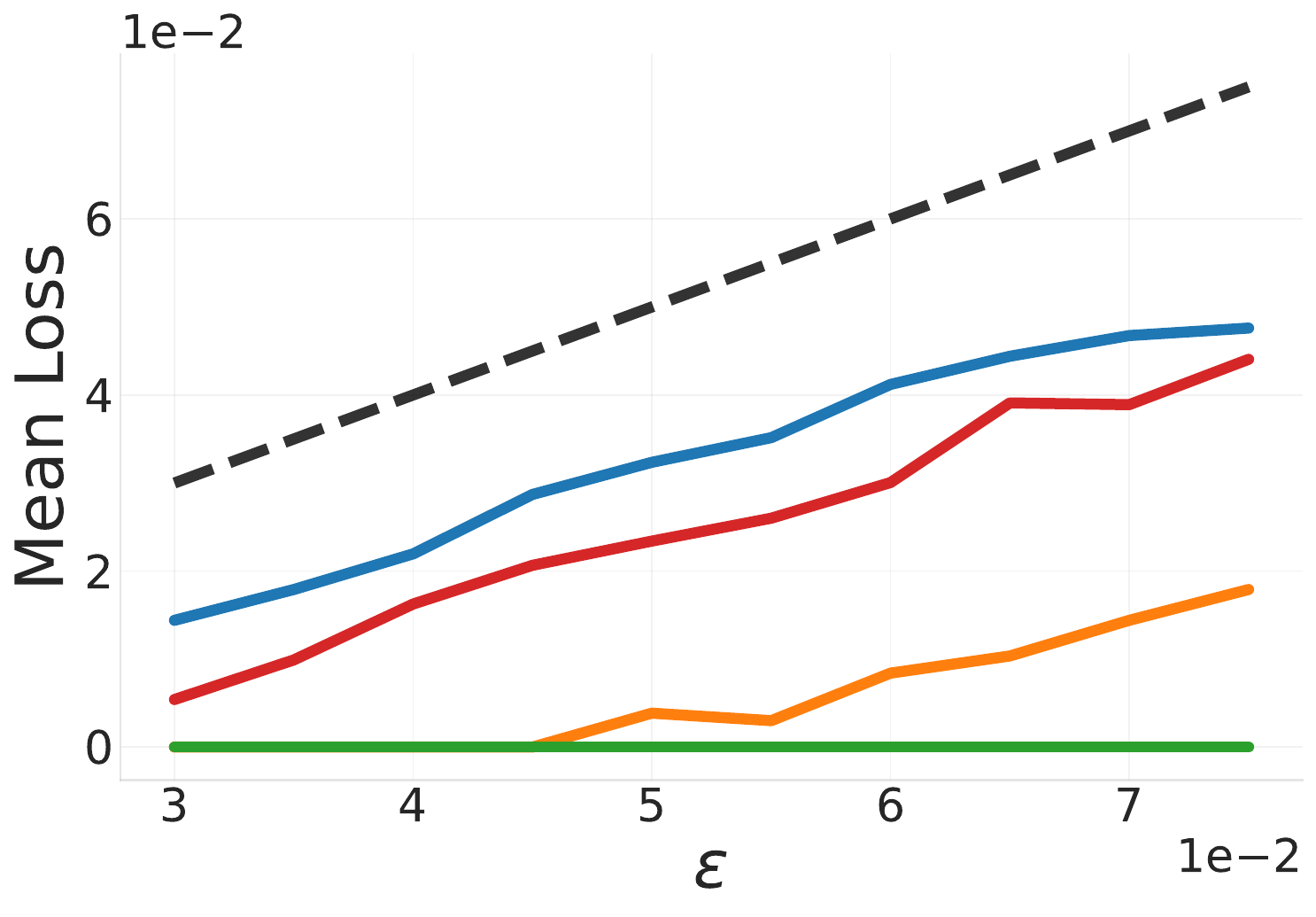}
        \caption*{Mean loss}
    \end{minipage}\hfill
    \begin{minipage}{0.32\linewidth}
        \centering
        \includegraphics[width=\linewidth]{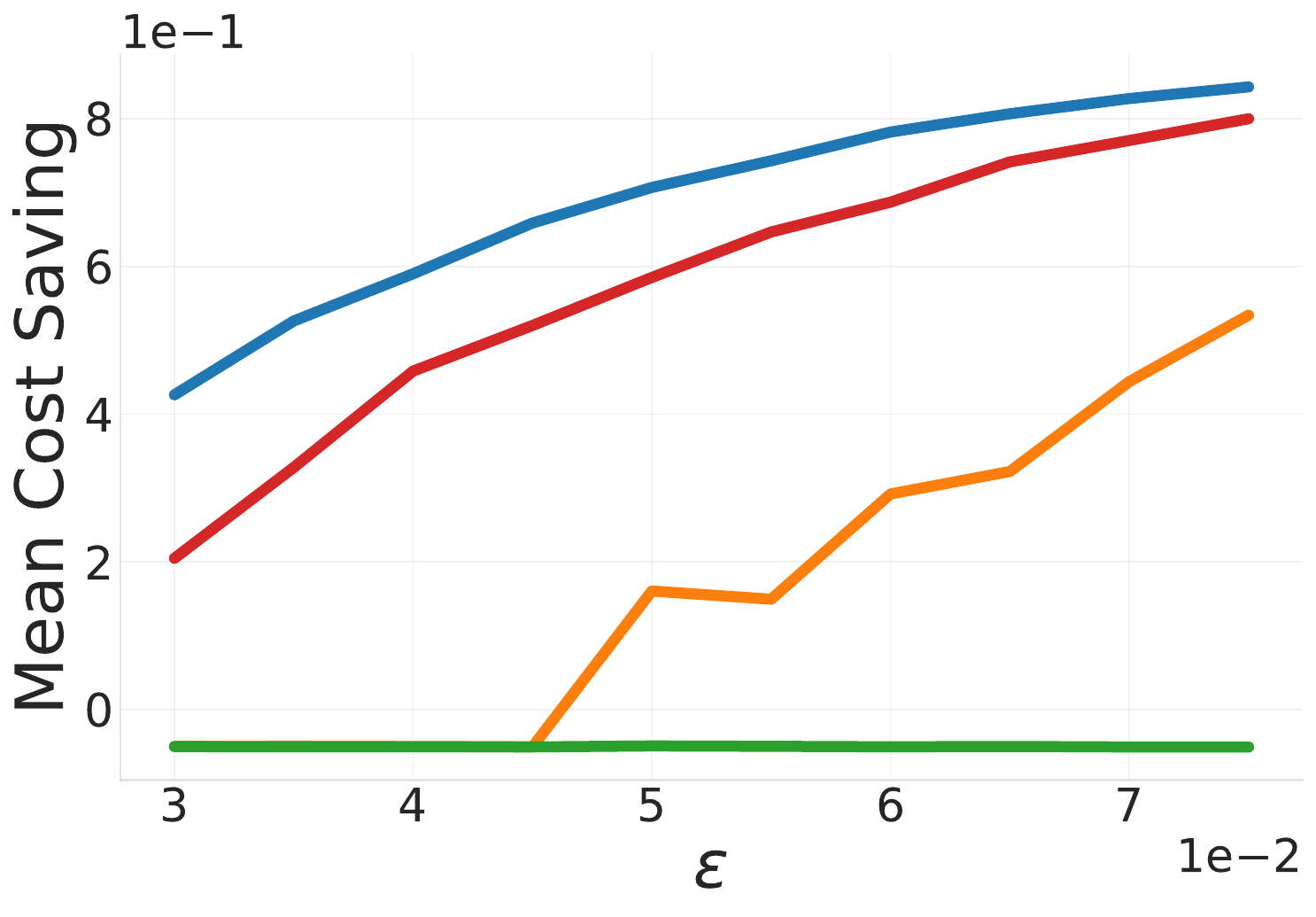}
        \caption*{Mean cost saving}
    \end{minipage}
    \caption{
    \textbf{The CLT-based UCB provides tighter bounds and better cost-performance tradeoff than the Hoeffding-based UCB.}
    Annotation performance of HyPAC under different upper confidence bounds. The experiments are conducted on the MATH-500 using Qwen3-4B-Instruct as the non-thinking model and Qwen3-4B-Thinking as the thinking model with logits-based uncertainty score under a confidence level of $\alpha=0.05$.
    }
    \label{fig:three-parallel}
\end{figure}

\subsection{Performance of HyPAC under distribution shift}
\label{subsection:distribution_shift}
We emphasize that the calibration set is sampled i.i.d. from a large unlabeled dataset, leaving the remaining unlabeled data for testing. 
Given a large unlabeled dataset, we first annotate a small subset as the calibration set, then apply HyPAC to the remaining data.
This process ensures that the calibration and test datasets are naturally i.i.d., satisfying the standard assumptions of HyPAC.
Thus, HyPAC generally does not encounter distribution shifts in annotation tasks.

Although our method does not provide theoretical guarantees under distribution shift, we evaluate the empirical robustness of HyPAC using ImageNet \citep{deng2009imagenet} and ImageNet-C (Brightness) \citep{hendrycks2019benchmarking}. 
Specifically, we use ImageNet as the calibration set and ImageNet-C Brightness with varying severity levels as the test set.
We conduct experiments on image annotation tasks instead of LLM open-ended generation tasks because there are currently no benchmark datasets with different levels of covariate shift for open-ended LLM generation, as far as we know.
We use ResNet-34 \citep{he2016deep} as the small model $\tilde{f}_1$, and ResNet-152 as the large model $\tilde{f}_2$.
We define the cost of using ResNet-34, ResNet-152, and expert annotation to be $c_1=1$, $c_2=2$, and $c_h=8$, respectively.
We use one minus the maximum softmax probability as the uncertainty score function.

We present the performance of HyPAC across various testing sets in Table \ref{tab:brightness_results}.
The results show that the cost saving of HyPAC is getting worse with a higher severity of distribution shift (may be due to the degraded accuracy), while the annotation error is relatively insensitive.
Across all the severity levels, HyPAC successfully controls the annotation error below the target level, demonstrating that HyPAC is empirically robust to distribution shift.

\begin{table}[t]
\centering
\caption{Performance of HyPAC under varying severity levels of ImageNet-C Brightness corruption. 
HyPAC maintains stable error control under different levels of shift, while the cost saving decreases with increasing corruption severity.}
\label{tab:brightness_results}
\footnotesize
\setlength{\tabcolsep}{4pt}
\begin{tabular}{lccccc}
\toprule
\multirow{2}{*}{Test Dataset} 
& \multicolumn{2}{c}{$\epsilon=0.05$} 
& \multicolumn{2}{c}{$\epsilon=0.1$} \\
\cmidrule(lr){2-3} \cmidrule(lr){4-5}
 & Error (\%) & Save (\%) & Error (\%) & Save (\%) \\
\midrule
Severity 1 & 4.51 & 40.75 & 9.75 & 55.19 \\
Severity 2 & 4.54  & 38.27 & 9.88 & 52.88 \\
Severity 3 & 4.52  & 34.25 & 9.87 & 48.99 \\
Severity 4 & 4.25 & 27.98 & 9.94 & 42.81 \\
Severity 5 & 4.11  & 20.69 & 9.78 & 34.96 \\
\bottomrule
\end{tabular}
\end{table}


\subsection{More annotation results across diverse models and datasets}
\label{subsec:additional_main_experiment}
Table~\ref{tab:results_of_other_llm} reports the annotation performance of HyPAC under two model pair configurations: 
(i) Llama-3.1-8B-Instruct as the non-thinking model paired with DeepSeek-R1-Distill-Llama-8B as the thinking model, and 
(ii) Qwen2.5-32B-Instruct as the non-thinking model paired with DeepSeek-R1-Distill-Qwen-32B as the thinking model.
HyPAC achieves the highest cost savings across most datasets.
However, on the Zebra-Logic dataset with the Llama-8B model pair, HyPAC incurs a higher annotation cost than using expert annotation alone.
This behavior arises because the Llama-8B model pair attains low accuracy on Zebra-Logic (13.2\% and 34.8\%), and HyPAC relies heavily on expert annotations to satisfy the accuracy constraint.
This behavior shows that the cost savings of HyPAC are closely tied to the capability of the employed LLMs.
Higher-capacity LLMs require fewer expert annotations to satisfy accuracy constraints, thus yielding greater cost savings.

\begin{table*}[t]
\centering
\caption{\textbf{HyPAC achieves high cost savings across different model pairs while maintaining controlled annotation error.}
Annotation performance of HyPAC across four benchmark datasets using two model pairs: 
Llama-3.1-8B-Instruct (non-thinking model $\tilde{f}_1$) with DeepSeek-R1-Distill-Llama-8B (thinking model $\tilde{f}_2$), and 
Qwen2.5-32B-Instruct (non-thinking model $\tilde{f}_1$) with DeepSeek-R1-Distill-Qwen-32B (thinking model $\tilde{f}_2$).
We report results for the annotation error and cost saving at $\epsilon=0.05$ and $\alpha=0.05$ using the logits-based uncertainty score function and token-based cost function.
\textbf{Bold} numbers indicate the best results.
}

\label{tab:results_of_other_llm}
\renewcommand\arraystretch{1.1}
\resizebox{1.00\textwidth}{!}{
\setlength{\tabcolsep}{5mm}{
\begin{tabular}{ll|ccc|ccc}
\toprule
\multirow{2}{*}{Datasets} 
& \multirow{2}{*}{Metric} 
& \multicolumn{3}{c}{Llama-8B model pair} 
& \multicolumn{3}{c}{Qwen-32B model pair} \\
\cmidrule(lr){3-5} \cmidrule(lr){6-8}
& 
& $\tilde{f}_1$ + $\tilde{f}_2$ + Human
& $\tilde{f}_1$ + Human
& $\tilde{f}_2$ + Human
& $\tilde{f}_1$ + $\tilde{f}_2$ + Human 
& $\tilde{f}_1$ + Human 
& $\tilde{f}_2$ + Human \\
\midrule
\multirow{2}{*}{MMLU-Redux}
& Error (\%)         
& 3.95 & 3.87 & 4.08 & 4.00 & 3.83 & 4.09 \\
& Save (\%)   
& \textbf{16.83} & 7.19 & 13.40 & 45.29 & 32.93 & \textbf{50.57} \\
\midrule
\multirow{2}{*}{MATH-500}
& Error (\%)       
& 3.51 & 2.82 & 3.14 & 3.96 & 3.46 & 3.50 \\
& Save (\%)    
& \textbf{10.25} & -5.37 & 10.18 & \textbf{56.89} & 20.49 & 36.35 \\
\midrule
\multirow{2}{*}{MATH-L5}
& Error (\%)       
& 3.21 & 3.17 & 3.55 & 3.70 & 3.24 & 3.32 \\
& Save (\%)    
& -5.39 & -12.03 & \textbf{11.38} & \textbf{72.67} & 9.75 & 58.80 \\
\midrule
\multirow{2}{*}{ZebraLogic}
& Error (\%)         
& 3.55 & 3.68 & 3.55 & 4.30 & 3.71 & 3.57 \\
& Save (\%)     
& -20.11 & \textbf{-19.94} & -25.14 & \textbf{6.35} & 4.10 & -9.14 \\
\midrule
\multirow{2}{*}{$\text{HumanEval}^{+}$}
& Error (\%)         
& 3.78 & 3.80 & 2.78 & 3.21 & 3.38 & 0.78 \\
& Save (\%)   
& \textbf{47.64} & 31.95 & 5.97 & \textbf{90.41} & 16.05 & 85.34 \\
\bottomrule
\end{tabular}
}
}
\vspace{-5pt}
\end{table*}

\subsection{More reliability diagrams of logits-based uncertainty score and verbalized uncertainty score}
\label{subsec:more_reliability}
We present the reliability diagrams of Qwen3-4B-Instruct-2507 on MMLU-redux, MATH-L5, Zebra-Logic and $\text{HumanEval}^+$ in Figure~\ref{fig:reliability_8panel}.
The results show that logits-based confidence achieves lower ECE than verbalized confidence in all datasets except for Zebra-Logic.

\begin{figure*}[t]
    \centering
    \includegraphics[width=1\textwidth]{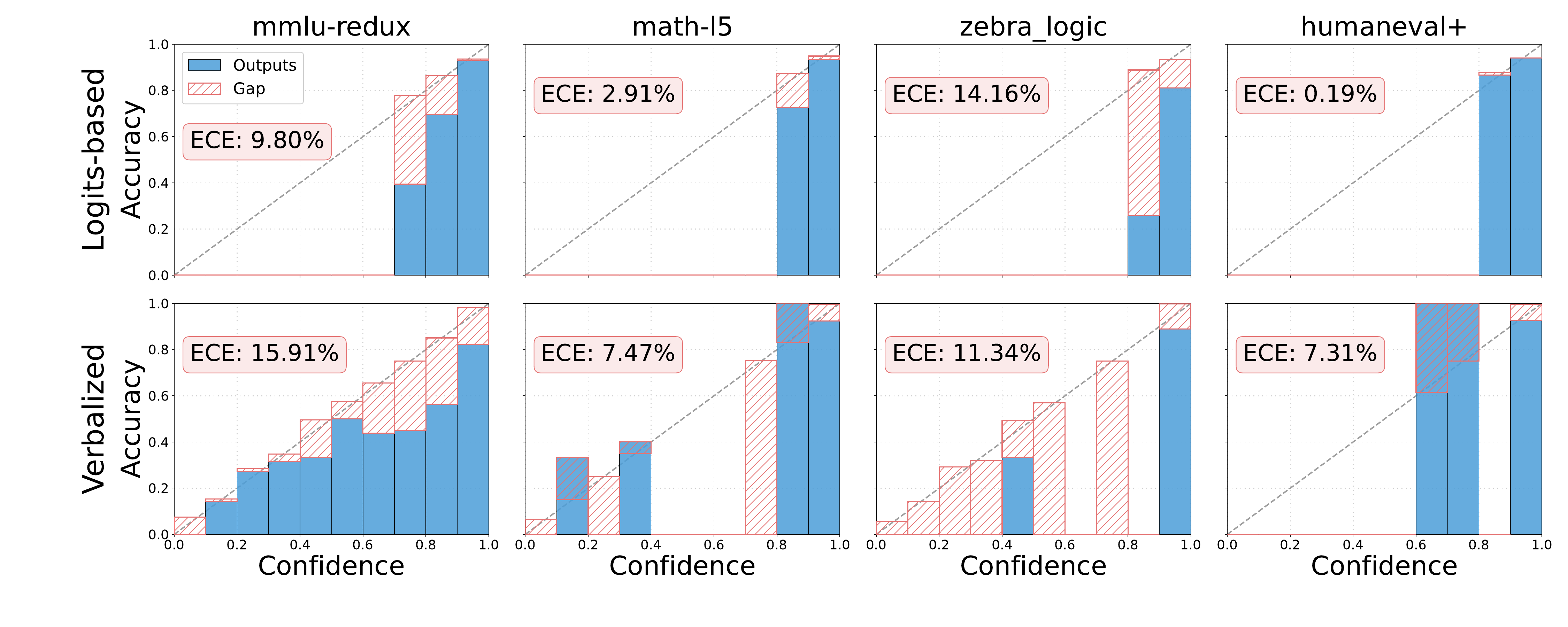}
    \caption{
        Reliability diagrams of Qwen3-4B-Instruct-2507 across four benchmark datasets
        (MMLU-Redux, MATH-L5, Zebra-Logic, and $\text{HumanEval}^+$).
        Each column corresponds to a dataset.
        The top row uses logits-based confidence, while the bottom row uses verbalized confidence.}
    \label{fig:reliability_8panel}
\end{figure*}

\end{document}